\documentclass{article}

\usepackage{times}
\usepackage{graphicx} % more modern
\usepackage{subfigure} 

\usepackage{natbib}

\usepackage{algorithm}
\usepackage{algorithmic}

\usepackage{hyperref}

\usepackage{jmlr2e1} 

\usepackage{amssymb,amsmath,amsfonts}%,amsthm}
\usepackage{mathrsfs}
\usepackage{wasysym}
\newcommand{\norm}[1]{\left\lVert#1\right\rVert}
\newcommand{\expect}[1]{\mathbb{E}\left[{#1}\right]}
\newcommand{\prob}[1]{\mathbb{P}\left[{#1}\right]}
\newcommand{\given}{\; \big\vert \;} 
\newcommand{\bydef}{:=}

\newcommand{\Otilde}[1]{\tilde{O}\left(#1\right)}

\newcommand{\inner}[2]{\langle #1, #2 \rangle}
\newcommand{\at}[2][]{#1|_{#2}}

\usepackage{subfigure}
\usepackage{mathnotations2}

\graphicspath{{./Figures/}}

\newtheorem{mytheorem}{Theorem}

\newtheorem{mylemma}{Lemma}

\newtheorem{mydefinition}{Definition}

\newcommand{\beq}{\begin{equation}}
\newcommand{\eeq}{\end{equation}}
\newcommand{\beqn}{\begin{equation*}}
\newcommand{\eeqn}{\end{equation*}}
\newcommand{\beqa}{\begin{eqnarray}}
\newcommand{\eeqa}{\end{eqnarray}}
\newcommand{\beqan}{\begin{eqnarray*}}
\newcommand{\eeqan}{\end{eqnarray*}}

\newcommand{\argmax}{\mathop{\mathrm{argmax}}}

\newcommand{\argmin}{\mathop{\mathrm{argmin}}}

\begin{document} 

\title{On Kernelized Multi-armed Bandits}
\author{\name{Sayak Ray Chowdhury} \email{srchowdhury@ece.iisc.ernet.in}\\ 
\addr      Electrical Communication Engineering,\\
Indian Institute of Science,\\
Bangalore 560012, India\\
\\
\name{Aditya Gopalan} \email{aditya@ece.iisc.ernet.in}\\ 
\addr     Electrical Communication Engineering,\\
Indian Institute of Science,\\
Bangalore 560012, India\\
}

%\editor{...}

\maketitle

\begin{abstract} 
%The purpose of this document is to provide both the basic paper template and
%submission guidelines. Abstracts should be a single paragraph, between 4--6 sentences long, ideally.  Gross violations will trigger corrections at the camera-ready phase.
We consider the stochastic bandit problem with a continuous set of arms, with the expected reward function over the arms assumed to be fixed but unknown. We provide two new Gaussian process-based algorithms for continuous bandit optimization -- Improved GP-UCB (IGP-UCB) and GP-Thomson sampling (GP-TS), and derive corresponding regret bounds. Specifically, the bounds hold when the expected reward function belongs to the reproducing kernel Hilbert space (RKHS) that naturally corresponds to a Gaussian process kernel used as input by the algorithms. Along the way, we derive a new self-normalized concentration inequality for vector-valued martingales of arbitrary, possibly infinite, dimension. Finally, experimental evaluation and comparisons to existing algorithms on synthetic and real-world environments are carried out that highlight the favorable gains of the proposed strategies in many cases. 
\end{abstract}
\section{Introduction}
\label{sec:Introduction}

% what is the problem and why the problem is relevant, broadly
Optimization over large domains under uncertainty is an important subproblem arising in a variety of sequential decision making problems, such as dynamic pricing in economics \citep{BesZee09:dynpricinglearning}, reinforcement learning with continuous state/action spaces \citep{kaelbling1996reinforcement, SmaKae00:practicalRLcontinuous}, and power control in wireless communication \citep{ChiHanLanTan08:powercontrol}. A typical feature of such problems is a large, or potentially infinite, domain of decision points or covariates (prices, actions, transmit powers), together with only partial and noisy observability of the associated outcomes (demand, state/reward, communication rate); reward/loss information is revealed only for decisions that are chosen. This often makes it hard to balance exploration and exploitation, as available knowledge must be transferred efficiently from a finite set of observations so far to estimates of the values of infinitely many decisions. A classic case in point is that of the canonical stochastic MAB with finitely many arms, where the effort to optimize scales with the total number of arms or decisions; the effect of this is catastrophic for large or infinite arm sets. 

% what is known about the problem (quick survey)
With suitable structure in the values or rewards of arms, however, the challenge of sequential optimization can be efficiently addressed. Parametric bandits, especially linearly parameterized bandits \citep{RusTsi10:linbandits}, represent a well-studied class of structured decision making settings. Here, every arm corresponds to a known, finite dimensional vector (its feature vector), and its expected reward is assumed to be an unknown linear function of its feature vector. This allows for a large, or even infinite, set of arms all lying in space of finite dimension, say $d$, and a rich line of work gives algorithms that attain sublinear regret with a polynomial dependence on the dimension, e.g., Confidence Ball \citep{dani2008stochastic}, OFUL \citep{abbasi2011improved} (a strengthening of Confidence Ball) and Thompson sampling for linear bandits \citep{agrawal2013thompson}\footnote{Roughly, for rewards bounded in $[-1,1]$, these algorithms achieve optimal regret $\Otilde{d\sqrt{T}}$, where $\Otilde{\cdot}$ hides $\text{polylog}(T)$ factors.} The insight here is that even though the number of arms can be large, the number of unknown parameters (or degrees of freedom) in the problem is really only $d$, which makes it possible to learn about the values of many other arms by playing a single arm. 

%using the framework of $\mathcal{X}$-armed bandits \cite{Xarmedbanditrefs}. A key insight is that when the target value function is sufficiently smooth (a natural property in many physical problem settings), one can effectively learn about the values of (infinitely many) neighboring arms by playing a particular arm, making it possible to attain regret that is sublinear in the number of rounds of optimization \cite{refs}. For instance, \cite{csaba-bubeck?} show that with a Lipschitzness assumption on the reward function, variants of the well-known UCB \cite{auer} bandit algorithm, originally designed for a finite arm set, adapted to run on a suitable discretization of the continuous domain, yield nontrivial regret guarantees.  it is impressive that ... can be attained even when the set of arms is infinite, and that the regret really depends on only ... 

A different approach to modelling bandit problems with a continuum of arms is via the framework of Gaussian processes (GPs) \citep{rasmussen2006gaussian}. GPs are a flexible class of nonparametric models for expressing uncertainty over functions on rather general domain sets, which generalize multivariate Gaussian random vectors. GPs allow tractable regression for estimating an unknown function given a set of (noisy) measurements of its values at chosen domain points. The fact that GPs, being distributions on functions, can also help quantify function uncertainty makes it attractive for basing decision making strategies on them. This has been exploited to great advantage to build nonparametric bandit algorithms, such as GP-UCB \citep{srinivas2009gaussian}, GP-EI and GP-PI \citep{hoffman2011portfolio}. In fact, GP models for bandit optimization, in terms of their kernel maps, can be viewed as the parametric linear bandit paradigm pushed to the extreme, where each feature vector associated to an arm can have infinite dimension \footnote{The completion of the linear span of all feature vectors (images of the kernel map) is precisely the reproducing kernel Hilbert space (RKHS) that characterizes the GP.}. 
%
%It is rather surprising, thus, that sublinear regret can be achieved even in such nonparametric (infinite dimensional) settings. 
%\todo{talk about X-armed bandits? Emphasize that GP models are a flexible toolkit -- Lipschitz alone is weak (gives bad rates) while linear in a known basis is too restrictive}

% Thompson sampling, finite-dim linear bandits, etc.
% ... ... 

% what this paper brings
Against this backdrop, our work revisits the problem of bandit optimization with stochastic rewards. Specifically, we consider stochastic multiarmed bandit (MAB) problems with a continuous arm set, and whose (unknown) expected reward function is assumed to lie in a reproducing kernel Hilbert space (RKHS), with bounded RKHS norm -- this effectively enforces smoothness on the function\footnote{Kernels, and their associated RKHSs, }. We make the following contributions-
\begin{itemize}
\item We design a new algorithm -- Improved Gaussian Process-Upper Confidence Bound (IGP-UCB) -- for stochastic bandit optimization. The algorithm can be viewed as a variant of GP-UCB \citep{srinivas2009gaussian}, but uses a significantly reduced confidence interval width resulting in an order-wise improvement in regret compared to GP-UCB. IGP-UCB also shows a markedly improved numerical performance over GP-UCB. 

\item We develop a nonparametric version of Thompson sampling, called Gaussian Process Thompson sampling (GP-TS), and show that enjoys a regret bound of $\Otilde{\gamma_T \sqrt{dT}}$. Here, $T$ is the total time horizon and $\gamma_T$ is a quantity depending on the RKHS containing the reward function. This is, to our knowledge, the first known regret bound for Thompson sampling in the agnostic setup with nonparametric structure. 

\item We prove a new self-normalized concentration inequality for infinite-dimensional vector-valued martingales, which is not only key to the design and analysis of the IGP-UCB and GP-TS algorithms, but also potentially of independent interest. The inequality generalizes a corresponding self-normalized bound for martingales in finite dimension proven by \citet{abbasi2011improved}. 
%Attempting to lift the result from finite to infinite dimension naively using the {\em method of mixtures} technique that they employ presents significant technical hurdles, as the self-normalization here corresponds to a matrix whose dimension grows with time as opposed to the fixed dimensional former setting. We overcome this hurdle by developing a method of mixtures construction in infinite dimension, induced by a full-blown Gaussian process.

\item Empirical comparisons of the algorithms developed above, with other GP-based algorithms, are presented, over both synthetic and real-world setups, demonstrating performance improvements of the proposed algorithms, as well as their performance under misspecification.

\end{itemize}

\section{Problem Statement}
\label{sec:Problem-Statement}
We consider the problem of sequentially maximizing a fixed but unknown reward function $f:D\ra \Real$ over a (potentially infinite) set of decisions $D \subset \Real^d$, also called actions or arms. An algorithm for this problem chooses, at each round $t$, an action $x_t \in D$, and subsequently observes a reward $y_t=f(x_t) + \epsilon_t$, which is a noisy version of the function value at $x_t$. The action $x_t$ is chosen causally depending upon the arms played and rewards obtained upto round $t-1$, denoted by the history $\cH_{t-1} = \lbrace (x_s,y_s): s = 1,\ldots,t-1\rbrace$. We assume that the noise sequence $\lbrace\epsilon_t\rbrace_{t=1}^{\infty}$ is conditionally $R$-sub-Gaussian for a fixed constant $R \ge 0$, i.e.,
\beq
\forall t \ge 0,\;\;\forall \lambda \in \Real, \;\; \expect{e^{\lambda \epsilon_t} \given \cF_{t-1}} \le \exp\left(\frac{\lambda^2R^2}{2}\right),
\label{eqn:noise}
\eeq
where $\cF_{t-1}$ is the $\sigma$-algebra generated by the random variables $\lbrace x_s, \epsilon_s\rbrace_{s = 1}^{t-1}$ and $x_t$.%
% defined as $\cF_{t-1} = \sigma \Big(\lbrace x_s, \epsilon_s\rbrace_{s = 1}^{t-1},x_{t}\Big)$.
This is a mild assumption on the noise (it holds, for instance, for distributions bounded in $[-R, R]$) and is standard in the bandit literature \citep{abbasi2011improved,agrawal2013thompson}.
 
%This assumption is satisfied whenever $y_t \in [f(x_t)-R,f(x_t)+R]$ and is prevalent in literature (see \citet{abbasi2011improved}, \citet{agrawal2013thompson} etc.). Also, it automatically implies that $\expect{\epsilon_t\given \cF_{t-1}} = 0$ and $\Var[\epsilon_t\given \cF_{t-1}] \le R^2$ and thus $R^2$ can be seen as the conditional variance of the noise.
\textbf{Regret.} The goal of an algorithm is to maximize its cumulative reward or alternatively minimize its cumulative {\em regret} -- the loss incurred due to not knowing $f$'s maximum point beforehand. Let $x^\star \in \argmax_{x\in D}f(x)$ be a maximum point of $f$ (assuming the maximum is attained). The instantaneous regret incurred at time $t$ is $r_t = f(x^\star)-f(x_t)$, and the cumulative regret in a time horizon $T$ (not necessarily known a priori) is defined to be $R_T = \sum_{t=1}^T r_t$. A sub-linear growth of $R_T$ in $T$ signifies that $R_T/T \ra 0$ as $T\ra \infty$, or vanishing per-round regret. 

\textbf{Regularity Assumptions.} Attaining sub-linear regret is impossible in general for arbitrary reward functions $f$ and domains $D$, and thus some regularity assumptions are in order. In what follows, we assume that $D$ is compact. The smoothness assumption we make on the reward function $f$ is motivated by Gaussian processes\footnote{Other work has also studied continuum-armed bandits with weaker smoothness assumptions such as Lipschitz continuity -- see Related work for details and comparison.} and their associated reproducing kernel Hilbert spaces (RKHSs, see \citet{scholkopf2002learning}). Specifically, we assume that $f$ has small norm in the RKHS of functions $D \to \mathbb{R}$, with positive semi-definite kernel function $k: D \times D \to \mathbb{R}$. This RKHS, denoted by $H_k(D)$, is completely specified by its kernel function $k(\cdot,\cdot)$ and vice-versa, with an inner product $\inner{\cdot}{\cdot}_k$ obeying the reproducing property: $f(x)=\inner{f}{k(x,\cdot)}_k$ for all $f \in H_k(D)$. In other words, the kernel plays the role of delta functions to represent the evaluation map at each point $x \in D$ via the RKHS inner product. The RKHS norm $\norm{f}_k = \sqrt{\inner{f}{f}}_k$ is a measure of the smoothness\footnote{One way to see this is that for every element $g$ in the RKHS, $|g(x) - g(y)| = |\inner{g}{k(x,\cdot) - k(y,\cdot)}| \le \norm{g}_k \norm{k(x,\cdot) - k(y,\cdot)}_k$ by Cauchy-Schwarz.} of $f$, with respect to the kernel function $k$, and satisfies: $f \in H_k(D)$ if and only if $\norm{f}_k < \infty$.

We assume a known bound on the RKHS norm of the unknown target function\footnote{This is analogous to the bound on the weight $\theta$ typically assumed in regret analyses of linear parametric bandits.}: $\norm{f}_k \le B$. Moreover, we assume bounded variance by restricting $k(x,x) \le 1$, for all $x \in D$. Two common kernels that satisfy  bounded variance property are \textit{Squared Exponential} and  \textit{Mat$\acute{e}$rn}, defined as
\beqan
k_{SE}(x,x') &=& \exp\Big(-s^2/2l^2\Big), \\
k_{Mat\acute{e}rn}(x,x') &=& \frac{2^{1-\nu}}{\Gamma(\nu)}\Big(\frac{s\sqrt{2\nu}}{l}\Big)^\nu B_\nu\Big(\frac{s\sqrt{2\nu}}{l}\Big),
\eeqan
where $l > 0$ and $\nu > 0$ are hyperparameters, $s = \norm{x-x'}_2$ encodes the similarity between two points $x,x'\in D$, and $B_\nu(\cdot)$ is the modified Bessel function. Generally the bounded variance property holds for any stationary kernel, i.e. kernels for which $k(x,x')=k(x-x')$ for all $x,x'\in \Real^d$. These assumptions are required to make the regret bounds scale-free and are standard in the literature \citep{agrawal2013thompson}. Instead if $k(x,x) \le c$ or $\norm{f}_k \le cB$, then our regret bounds would increase by a factor of $c$.

%\footnote{Finite domains are included here, and in fact all of our results are fact simpler to obtain in this case compared to the compact case.}

%\textit{\textbf{Remark.}} Notice that $f(x^\star)-\max_{t=1}^T f(x_t)$ is upper bounded by the average cumulative regret $R_T/T = f(x^\star) - \frac{1}{T}\sum_{t=1}^{T}f(x_t)$ and thus a bound on $R_T/T$ in turn provides a bound on the convergence rate of the optimization problem.
\section{Algorithms}
\label{sec:Algorithms}
{\bf Design philosophy.} Both the algorithms we propose use Gaussian likelihood models for observations, and Gaussian process (GP) priors for uncertainty over reward functions. A Gaussian process over $D$, denoted by $GP_D(\mu(\cdot),k(\cdot,\cdot))$, is a collection of random variables $(f(x))_{x \in D}$, one for each $x \in D$, such that every finite sub-collection of random variables $(f(x_i))_{i = 1}^m$ is jointly Gaussian with mean $\expect{f(x_i)} = \mu(x_i)$ and covariance $\expect{(f(x_i)-\mu(x_i))(f(x_j)-\mu(x_j))} = k(x_i, x_j)$, $1 \le i, j \le m$, $m \in \mathbb{N}$. The algorithms use $GP_D(0,v^2k(\cdot,\cdot))$, $v > 0$, as an initial prior distribution for the unknown reward function $f$ over $D$, where $k(\cdot,\cdot)$ is the kernel function associated with the RKHS $H_k(D)$ in which $f$ is assumed to have `small' norm at most $B$. The algorithms also assume that the noise variables $\epsilon_t = y_t-f(x_t)$ are drawn independently, across $t$, from $\cN(0,\lambda v^2)$, with $\lambda \ge 0$. Thus, the prior distribution for each $f(x)$, is assumed to be $\cN(0,v^2k(x,x))$, $x\in D$. Moreover, given a set of sampling points $A_t = (x_1, \ldots, x_t)$ within $D$, it follows under the assumption that the corresponding vector of observed rewards $y_{1:t}= [y_1,\ldots,y_t]^T$ has the multivariate Gaussian distribution $\cN(0,v^2(K_t+\lambda I))$, where $K_t = [k(x,x')]_{x,x'\in A_t}$ is the kernel matrix at time $t$. Then, by the properties of GPs, we have that $y_{1:t}$ and $f(x)$ are jointly Gaussian given $A_t$:
\beqn
\begin{bmatrix}
f(x) \\
y_{1:t} 
\end{bmatrix} \sim \cN\left(0,\begin{bmatrix}
v^2k(x,x) & v^2k_t(x)^T\\
v^2k_t(x) & v^2(K_t+\lambda I)
\end{bmatrix}\right),
\eeqn
where $k_t(x) = [k(x_1,x),\ldots, k(x_t,x)]^T$. Therefore conditioned on the history $\cH_t$, the posterior distribution over $f$ is  $GP_D(\mu_t(\cdot),v^2k_t(\cdot,\cdot))$, where
\beqa
\mu_t(x) &=& k_t(x)^T(K_t + \lambda I)^{-1}y_{1:t}, \label{eqn:mean update}\\
k_t(x,x') &=& k(x,x') - k_t(x)^T(K_t + \lambda I)^{-1} k_t(x'),\label{eqn:cov update}\\
\sigma_t^2(x) &=& k_t(x,x).\label{eqn:sd update}
\eeqa
Thus for every $x \in D$, the posterior distribution of $f(x)$, given $\cH_t$, is $\cN(\mu_t(x),v^2\sigma_t^2(x))$.

\textit{\textbf{Remark.}} Note that the GP prior and Gaussian likelihood model described above is only an aid to algorithm design, and has nothing to do with the actual reward distribution or noise model as in the problem statement (Section \ref{sec:Problem-Statement}). The reward function $f$ is a fixed, unknown, member of the RKHS $H_k(D)$, and the true sequence of noise variables $\epsilon_t$ is allowed to be a  conditionally $R$-sub-Gaussian martingale difference sequence (Equation \ref{eqn:noise}). In general, thus, this represents a misspecified prior and noise model, also termed the {\em agnostic} setting by \citet{srinivas2009gaussian}.

The proposed algorithms, to follow, assume the knowledge of only the sub-Gaussianity parameter $R$, kernel function $k$ and upper bound $B$ on the RKHS norm of $f$. Note that $v, \lambda$ are free parameters (possibly time-dependent) that can be set specific to the algorithm.

% \subsection{Upper Confidence Bound Based Algorithm}

\subsection{Improved GP-UCB (IGP-UCB) Algorithm}
\label{subesc:UCB}
%First we briefly describe two UCB based approaches from the literature designed to solve the problem at hand. Then we state our algorithm and draw a connection with both of them.\\
We introduce the IGP-UCB algorithm (Algorithm \ref{algo:ucb}), that uses a combination of the current posterior mean $\mu_{t-1}(x)$ and standard deviation $v \sigma_{t-1}(x)$ to (a) construct an upper confidence bound (UCB) envelope for the actual function $f$ over $D$, and (b) choose an action to maximize it. Specifically it chooses, at each round $t$, the action
\beq
x_t=\argmax_{x\in D}\mu_{t-1}(x)+\beta_t\sigma_{t-1}(x),
\eeq
\label{eqn:UCB-rule}
with the scale parameter $v$ set to be $1$.
Such a rule trades off exploration (picking points with high uncertainty $\sigma_{t-1}(x)$) with exploitation (picking points with high reward $\mu_{t-1}(x)$), with $\beta_t=B+ R\sqrt{2(\gamma_{t-1}+1+\ln(1/\delta))}$ being the parameter governing the tradeoff, which we later show is related to the width of the confidence interval for $f$ at round $t$. $\delta \in (0,1)$ is a free {\em confidence} parameter used by the algorithm, and $\gamma_t$ is the \textit{maximum information gain} at time $t$, defined as:
\beqn
\gamma_t \bydef \max_{A \subset D : \abs{A}=t} I(y_A;f_A).
\eeqn
Here, $I(y_A;f_A)$ denotes the \textit{mutual information} between $f_A = [f(x)]_{x\in A}$ and $y_A = f_A + \epsilon_A$, where $\epsilon_A \sim \cN(0,\lambda v^2 I)$ and quantifies the reduction in uncertainty about $f$ after observing $y_A$ at points $A \subset D$. $\gamma_t$ is a problem dependent quantity and can be found given the knowledge of domain $D$ and kernel function $k$. For a compact subset $D$ of $\Real^d$, $\gamma_T$ is $O((\ln T)^{d+1})$ and $O(T^{d(d+1)/(2\nu+d(d+1))}\ln T)$, respectively, for the Squared Exponential and Mat$\acute{e}$rn kernels \citep{srinivas2009gaussian}, depending only polylogarithmically on the time $T$. 
%This result, along with Theorems \ref{thm:regret-bound-UCB} and \ref{thm:regret-bound-TS}, guarantees sublinear regret of IGP-UCB and GP-TS for both kernels. \todo{$\gamma_t$ needs to be defined more clearly.}  
%The pseudo-code is shown in Algorithm \ref{algo:ucb}. 
%We modify GP-UCB by replacing $\hat{\mu}_t(x)$, $\hat{\sigma}_t(x)$ and $\hat{\beta}_t$ with $\mu_t(x)$, $\sigma\sigma_t(x)$ and $\beta_t/\sigma$ respectively, where $\beta_t$ is set to be  Thus at round $t$, we play the action that maximizes $\mu_{t-1}(x)+\beta_t\sigma_{t-1}(x)$. 
% as it reduces width of the confidence interval of GP-UCB by a multiplicative factor of $O(\ln^{3/2}(t))$ and hence enjoys an \textit{improved} regret bound of the same factor (Theorem \ref{thm:regret-bound-UCB}).

\begin{algorithm}
\renewcommand\thealgorithm{1}
\caption{Improved-GP-UCB (IGP-UCB)}\label{algo:ucb}
\begin{algorithmic}
\STATE \textbf{Input:} Prior $GP(0,k)$, parameters $B$, $R$, $\lambda$, $\delta$.
\FOR{t = 1, 2, 3 \ldots T}
\STATE Set $\beta_t=B+ R\sqrt{2(\gamma_{t-1}+ 1+\ln(1/\delta))}$.
\STATE Choose $x_t = \argmax\limits_{x\in D} \mu_{t-1}(x) + \beta_t\sigma_{t-1}(x)$.
\STATE Observe reward $y_t = f(x_t)+\epsilon_t$.
\STATE Perform update to get $\mu_t$ and $\sigma_t$ using \ref{eqn:mean update}, \ref{eqn:cov update} and \ref{eqn:sd update}.
\ENDFOR
\end{algorithmic}
\addtocounter{algorithm}{-1}
\end{algorithm}
\textit{\textbf{Discussion.}} 
\cite{srinivas2009gaussian} have proposed the GP-UCB algorithm, and \citet{valko2013finite} the KernelUCB algorithm, for sequentially optimizing reward functions lying in the RKHS $H_k(D)$. 
Both algorithms play an arm at time $t$ using the rule: $x_t=\argmax_{x\in D}\mu_{t-1}(x)+\tilde{\beta}_t\sigma_{t-1}(x)$.
%\footnotesize
%\beqan
%x_t &=& \argmax_{x\in D}
%%x_t &=& \argmax_{x\in D}\; k_{t-1}(x)^T(K_{t-1} + \lambda I)^{-1}y_{1:t-1} +\\&& \tilde{\beta}_t\sqrt{k(x,x) - k_{t-1}(x)^T(K_{t-1} + \lambda I)^{-1} k_{t-1}(x)}.
%\eeqan
%\normalsize
%\todo{For easy comparison, rewrite above eqn. in terms of $\mu_{t-1}$ and $\sigma_{t-1}$?}
GP-UCB uses the exploration parameter $\tilde{\beta}_t = \sqrt{2B^2+300\gamma_{t-1}\ln^3(t/\delta)}$, with $\lambda$ set to $\sigma^2$, where $\sigma$ is additionally assumed to be a known, uniform (i.e., almost-sure) upper bound on all noise variables $\epsilon_t$ \citep[Theorem $3$]{srinivas2009gaussian}. Compared to GP-UCB, IGP-UCB (Algorithm \ref{algo:ucb}) reduces the width of the confidence interval by a factor roughly $O(\ln^{3/2}t)$ at every round $t$, and, as we will see, this small but critical adjustment leads to much better theoretical and empirical performance compared to GP-UCB. In KernelUCB, $\tilde{\beta}_t$ is set as $\eta/\lambda^{1/2}$, where $\eta$ is the exploration parameter and $\lambda$ is the regularization constant. Thus IGP-UCB can be viewed as a special case of KernelUCB where $\eta=\beta_t$.

%where $\hat{\mu}_{t-1}(\cdot)$ and $\hat{\sigma}_{t-1}(\cdot)$ are the posterior mean and standard deviation of $f(x)$, starting with the prior $GP_D(0,k(\cdot,\cdot))$:
%\footnotesize
%\beqan
%\hat{\mu}_t(x) &=& k_t(x)^T(K_t + \sigma^2I)^{-1}y_{1:t},\\
%\hat{\sigma}^2_t(x) =  &=& k(x,x) - k_t(x)^T(K_t + \sigma^2I)^{-1} k_t(x).
%\eeqan
%\normalsize

%depending upon whether the reward function is sampled from GP or the RKHS, which in the latter case (exactly our problem setting) turns out to be
%GP-UCB obtain a regret bound  $O\Big(\sqrt{T\hat{\beta}_T\gamma_T}\Big)$ with probability at least $1-\delta.$

%\textbf{KernelUCB.} 
%\citet{valko2013finite} proposed KernelUCB for maximizing reward functions lying in the RKHS $H_k(D)$, where the decision set $D$ is \textit{finite}, by directly kernelising the LinUCB algorithm given in \citet{chu2011contextual}. 
%First, they solve a ridge regression problem in the RKHS and then use the kernelised version of the Mahalonbis distance to arrive at the following decision at round $t$:
%\footnotesize
%\beqan
%x_t &=& \argmax_{x\in D}\; k_t(x)^T(K_t + \lambda I)^{-1}y_{1:t} +\\&& \frac{\eta}{\lambda^{1/2}}\sqrt{k(x,x) - k_t(x)^T(K_t + \lambda I)^{-1} k_t(x)},
%\eeqan
%\normalsize

% \subsection{Thompson Sampling Based Algorithm}
\subsection{Gaussian Process Thompson Sampling (GP-TS)}

\label{subsection:TS}
%The Thompson Sampling (TS) algorithm assumes a prior on the underlying parameters of the reward distribution of every action and at each round, plays an arm according to the posterior probability that it is optimal. 
% ituses GP's as the prior distribution over unknown reward function $f$ and simply

Our second algorithm, GP-TS (Algorithm \ref{algo:ts}), inspired by the success of Thompson sampling for standard and parametric bandits \citep{agrawal2012analysis,kaufmann2012thompson,gopalan2014thompson,agrawal2013thompson}, uses the time-varying scale parameter $v_t = B + R\sqrt{2(\gamma_{t-1}+ 1+ \ln(2/\delta))}$ and operates as follows. At each round $t$, GP-TS samples a random function $f_t(\cdot)$ from the GP with mean function $\mu_{t-1}(\cdot)$ and covariance function $v_t^2k_{t-1}(\cdot,\cdot)$. Next, it chooses a decision set $D_t \subset D$, and plays the arm $x_t \in D_t$ that maximizes $f_t$\footnote{If $D_t = D$ for all $t$, then this is simply {\em exact} Thompson sampling. For technical reasons, however, our regret bound is valid when $D_t$ is chosen as a suitable discretization of $D$, so we include $D_t$ as an algorithmic parameter.}. We call it GP-Thompson-Sampling as it falls under the general framework of Thompson Sampling, i.e., (a) assume a prior on the underlying parameters of the reward distribution, (b)  play the arm according to the prior probability that it is optimal, and (c) observe the outcome and update the prior. However, note that the prior is nonparametric in this case. 
% The pseudo-code is shown in Algorithm \ref{algo:ts}.
%\footnote{Thus we set $\sigma$ to be $v_1 = B + R\sqrt{2 \ln(2/\delta)}$ here.}
%Then after observing reward $y_t$, we update the posterior as in Equation \ref{eqn:mean update} and \ref{eqn:cov update}.
\begin{algorithm}
\renewcommand\thealgorithm{2}
\caption{GP-Thompson-Sampling (GP-TS)}\label{algo:ts}
\begin{algorithmic}
\STATE \textbf{Input:} Prior $GP(0,k)$, parameters $B$, $R$, $\lambda$, $\delta$.
\FOR{t = 1, 2, 3 \ldots,}
\STATE Set $v_t = B + R\sqrt{2(\gamma_{t-1}+ 1+ \ln(2/\delta))}$.
\STATE Sample $f_t(\cdot)$ from $GP_D(\mu_{t-1}(\cdot),v_t^2k_{t-1}(\cdot,\cdot))$.
\STATE Choose the current decision set $D_t \subset D$. 
\STATE Choose $x_t = \argmax\limits_{x\in D_t} f_t(x)$.
\STATE Observe reward $y_t = f(x_t)+\epsilon_t$.
\STATE Perform update to get $\mu_t$ and $k_t$ using \ref{eqn:mean update} and \ref{eqn:cov update}.
\ENDFOR
\end{algorithmic}
\addtocounter{algorithm}{-2}
\end{algorithm}

\section{Main Results}
\label{sec:main-results}
We begin by presenting two key concentration inequalities which are essential in bounding the regret of the proposed algorithms.
\begin{mytheorem}
\label{thm:self-normalized-bound} 
Let $\lbrace x_t \rbrace_{t=1}^{\infty}$ be an $\Real^d$-valued discrete time stochastic process predictable with respect to the filtration $\lbrace \cF_t \rbrace_{t=0}^{\infty}$, i.e., $x_t$ is $\cF_{t-1}$-measurable $\forall t \ge 1$.  Let $\lbrace\epsilon_t\rbrace_{t=1}^{\infty}$ be a real-valued stochastic process such that for some $R \ge 0$ and for all $t \ge 1$, $\epsilon_t$ is (a) $\cF_t$-measurable, and (b) $R$-sub-Gaussian conditionally on $\cF_{t-1}$. Let $k: \Real^d \times \Real^d \to \Real$ be a symmetric, positive-semidefinite kernel, and let $0 < \delta \le 1$.
%\begin{enumerate}
%\item
For a given $\eta > 0$, with probability at least $1 -\delta$, the following holds simultaneously over all $t \ge 0$: 
\begin{equation}
\label{eqn:thmpart1}
\norm{\epsilon_{1:t}}_{((K_t+ \eta I)^{-1} + I)^{-1}}^2 \le  2R^2\ln \frac{\sqrt{\det((1+\eta)I+K_t)}}{\delta}.  
\end{equation}
(Here, $K_t$ denotes the $t \times t$ matrix $K_t(i,j) = k(x_i, x_j)$, $1 \le i, j \le t$ and for any $x \in \Real^t$ and $A \in \Real^{t\times t}$, $\norm{x}_A \bydef \sqrt{x^TAx}$).
%\footnotesize
%\beqn
%\label{eqn:thmpart1}
%\norm{\epsilon_{1:t}}_{((K_t+ \eta I)^{-1} + I)^{-1}}^2 \le  2\ln \frac{\sqrt{\det((1+\eta)I+K_t)}}{\delta}.  
%\eeqn
%\normalsize
%\item
Moreover, if $K_t$ is positive definite $\forall t \ge 1$ with probability 1, then the conclusion above holds with $\eta = 0$. 
%Moreover, if the kernel $k$ has associated with it the feature map $\phi: \Real^d \to H$ together with the reproducing kernel Hilbert space\footnote{Such a pair $(\phi, H)$ always exists, see e.g., \citet{rasmussen2006gaussian}.} (RKHS) $H$, then defining $S_t = \sum_{s=1}^{t}\epsilon_s\phi(x_s)$ and the (possibly infinite dimensional) matrix\footnote{More formally, $V_t: H \to H$ is the linear operator defined by $V_t(z) = z + \sum_{s=1}^t \phi(x_s) \inner{\phi(x_s)}{z}$ $\forall z \in H$.} $V_t = I+\sum_{s=1}^{t}\phi(x_s)\phi(x_s)^T$, we have
%\footnotesize
%\beqn
%\norm{\epsilon_{1:t}}_{\left( K_t^{-1} + I \right)^{-1}} = \norm{S_t}_{V_t^{-1}}, 
%\eeqn
%\normalsize
%where $\norm{S_t}_{V_t^{-1}}$ is ...
%\end{enumerate}
\end{mytheorem}

%\textbf{Discussion.} 
Theorem \ref{thm:self-normalized-bound} represents a self-normalized concentration inequality: the `size' of the increasing-length sequence $\{\epsilon_t\}_t$ of martingale differences is normalized by the growing quantity $((K_t+ \eta I)^{-1} + I)^{-1}$ that explicitly depends on the sequence. The following lemma helps provide an alternative, abstract, view of the self-normalized process of Theorem \ref{thm:self-normalized-bound}, based on the feature space representation induced by a kernel. 
\begin{mylemma}
\label{lem:selfnorm}
Let $k: \Real^d \times \Real^d \to \Real$ be a symmetric, positive-semidefinite kernel, with associated feature map $\phi: \Real^d \to H_k$ and the reproducing kernel Hilbert space\footnote{Such a pair $(\phi, H_k)$ always exists, see e.g., \citet{rasmussen2006gaussian}.} (RKHS) $H_k$. Letting $S_t = \sum_{s=1}^{t}\epsilon_s\phi(x_s)$ and the (possibly infinite dimensional) matrix\footnote{More formally, $V_t: H_k \to H_k$ is the linear operator defined by $V_t(z) = z + \sum_{s=1}^t \phi(x_s) \inner{\phi(x_s)}{z}$ $\forall z \in H_k$.} $V_t = I+\sum_{s=1}^{t}\phi(x_s)\phi(x_s)^T$, we have, whenever $K_t$ is positive definite, that
\beqn
\norm{\epsilon_{1:t}}_{\left( K_t^{-1} + I \right)^{-1}} = \norm{S_t}_{V_t^{-1}}, 
\eeqn
where $ \norm{S_t}_{V_t^{-1}} := \norm{V_t^{-1/2} S_t}_{H_k}$ denotes the norm of $V_t^{-1/2} S_t$ in the RKHS $H_k$. 
\end{mylemma}

%Note that in the above theorem $dim(K_t)=t\times t \ra \infty$ as $t\ra \infty$. This is in sharp contrast to the finite dimensional result of \citet{abbasi2011improved}, who show ...
%but here $dim(K_t)=d\times d = O(1)$.
%If we consider the $\sigma$-algebra $\cF_{t} = \sigma \Big(\lbrace x_s, \epsilon_s\rbrace_{s = 1}^{t},x_{t+1}\Big)$, then $\epsilon_t$ is $\cF_t$ measurable, $x_t$ is $\cF_{t-1}$ measurable. 
Observe that $S_t$ is $\cF_{t}$-measurable and also $\expect{S_t\given \cF_{t-1}} = S_{t-1}$. The process $\lbrace S_t \rbrace_{t\ge0}$ is thus a martingale with values\footnote{We ignore issues of measurability here.} in the RKHS $H$, which can possibly be infinite-dimensional, and moreover, whose deviation is measured by the norm weighted by $V_t^{-1}$, which is itself derived from $S_t$. Theorem \ref{thm:self-normalized-bound} represents the kernelized generalization of the finite-dimensional result of \citet{abbasi2011improved}, and we recover their result under the special case of a linear kernel: $\phi(x)=x$ for all $x \in \Real^d$.

We remark that when $\phi$ is a mapping to a finite-dimensional Hilbert space, the argument of \citet[Theorem 1]{abbasi2011improved} can be lifted to establish Theorem \ref{thm:self-normalized-bound}, but it breaks down in the generalized, infinite-dimensional RKHS setting, as the self-normalized bound in their paper has an explicit, growing dependence on the feature dimension. Specifically, the method of mixtures \citep{de2009self} or Laplace method, as dubbed by \citet{maillard2016self} (Lemma 5.2), fails to hold in infinite dimension. The primary reason for this is that the mixture distribution for finite dimensional spaces can be chosen independently of time, but in a nonparametric setup like ours, where the dimensionality of the self-normalizing factor $\left( K_t^{-1} + I \right)^{-1}$ itself grows with time, the use of (random) stopping times, precludes using time-dependent mixtures. We get around this difficulty by applying a novel `double mixture' construction, in which a pair of mixtures on (a) the space of real-valued functions on $\Real^d$, i.e., the support of a Gaussian process, and (b) on real sequences is simultaneously used to obtain a more general result, of potentially independent interest (see Section \ref{sec:Key-Techniques} and the appendix for details).  
%
% with Gaussian Processes as a mixing distribution over real-valued functions on $D \subset \Real^d$ and subsequently using a change of measure (induced ?) argument (see Section \ref{sec:Key-Techniques} for details).\todo{Emphasize the "double mixture" technique used in the proof}
%\textbf{Implications.}  , which may be of independent interest on its own.This can be shown by observing that $\norm{\eps_{1:t}}_{(K_t^{-1}+I)^{-1}}= \norm{S_t}_{(\Phi_t^T\Phi_t+I)^{-1}}$\footnote{This is achieved by appropriately denoting the RKHS inner product in matrix notations (see Section \ref{sec:Key-Techniques} for details).}, where $\Phi_t \bydef\big[k(x_1,\cdot)^T,\hdots,k(x_t,\cdot)^T\big]^T$ and $S_t \bydef \Phi_t^T\epsilon_{1:t} = \sum_{s=1}^{t}\epsilon_sk(x_s,\cdot)$. We get around this difficulty by applying the ``method of mixtures'' technique (\citet{de2009self}) with Gaussian Processes as a mixing distribution over real-valued functions on $D \subset \Real^d$ and subsequently using a change of measure argument (see Section \ref{sec:Key-Techniques} for details).

Our next result shows that how the posterior mean is concentrated around the unknown reward function $f$.
\begin{mytheorem}
\label{thm:true-function-bound}
Under the same hypotheses as those of Theorem \ref{thm:self-normalized-bound}, let $D \subset \Real^d$, and $f:D \to \Real$ be a member of the RKHS of real-valued functions on $D$ with kernel $k$, with RKHS norm bounded by $B$. Then, with probability at least $1-\delta$, the following holds for all $x \in D$ and $t \ge 1$:
\beqan
\abs{\mu_{t-1}(x)-f(x)}\le \Big(B + R\sqrt{2(\gamma_{t-1}+1+ \ln(1/\delta))}\Big)\sigma_{t-1}(x),
\eeqan
where $\gamma_{t-1}$ is the maximum information gain after $t-1$ rounds and $\mu_{t-1}(x)$, $\sigma^2_{t-1}(x)$ are mean and variance of posterior distribution defined as in Equation \ref{eqn:mean update}, \ref{eqn:cov update}, \ref{eqn:sd update}, with $\lambda$ set to $1+\eta$ and $\eta=2/T$.
\end{mytheorem}
%\todo{define $\mu_{t-1}$, $\gamma_{t-1}$}
%\textbf{Improvement over Kernel Least-squares \citep{maillard2016self}.} 
%Theorem $6$ of \citet{srinivas2009gaussian}, Lemma $1$ of \citet{valko2013finite} and Theorem $3.5$ of \citet{maillard2016self}, all of these state a similar result on the estimation of the unknown reward function from RKHS.\\
Theorem $3.5$ of \citet{maillard2016self} states a similar result on the estimation of the unknown reward function from the RKHS.
We improve upon it in the sense that the confidence bound in Theorem \ref{thm:true-function-bound} is {\em simultaneous} over all $x \in D$, while the bound has been shown only for a single, fixed $x$ in the Kernel Least-squares setting. %
%The reason is that, it uses the Laplace method directly to the scalar term $k_\tau(x)^T(K_\tau+I)^{-1}\epsilon_{1:\tau}$, owing to the fact that the method may not be applied to $(K_\tau+I)^{-1}\epsilon_{1:\tau}$, where the dimension, $\tau$, of this vector is a random stopping time.
We are able to achieve this result by virtue of Theorem \ref{thm:self-normalized-bound}.

%\todo{Should we discuss the computational effort required to optimize a general function per round?}

\subsection{Regret Bound of IGP-UCB}
\label{subsubsec:regret-UCB}
\begin{mytheorem}
\label{thm:regret-bound-UCB}
Let $\delta\in(0,1)$, $\norm{f}_k \le B$ and $\epsilon_t$ is conditionally $R$-sub-Gaussian. Running IGP-UCB for a function $f$ lying in the RKHS $H_k(D)$, we obtain a regret bound of $O\Big(\sqrt{T}(B\sqrt{\gamma_T}+\gamma_T)\Big)$ with high probability. More precisely, with probability at least $1-\delta$, $R_T = O\Big(B\sqrt{T\gamma_T}+\sqrt{T\gamma_T(\gamma_T+\ln(1/\delta))}\Big)$.
\end{mytheorem}

\textbf{Improvement over GP-UCB.} \citet{srinivas2009gaussian}, in the course of analyzing the GP-UCB algorithm, show that when the reward function lies in the RKHS $H_k(D)$, GP-UCB obtains regret $O\Big(\sqrt{T}(B\sqrt{\gamma_T}+\gamma_T\ln^{3/2}(T))\Big)$ with high probability (see Theorem $3$ therein for the exact bound). Furthermore, they assume that the noise $\epsilon_t$ is {\em uniformly bounded} by $\sigma$, while our sub-Gaussianity assumption (see Equation \ref{eqn:noise}) is slightly more general, and we are able to obtain a $O(\ln^{3/2}T)$ multiplicative factor improvement in the final regret bound thanks to the new self-normalized inequality (Theorem \ref{thm:self-normalized-bound}). Additionally, in our numerical experiments, we observe a significantly improved performance of IGP-UCB over GP-UCB, both on synthetically generated function, and on real-world sensor measurement data (see Section \ref{sec:Experiments}).

\textbf{Comparison with KernelUCB.} \citet{valko2013finite} show that the cumulative regret of KernelUCB is $\tilde{O}(\sqrt{\tilde{d}T})$, where $\tilde{d}$, defined as the \textit{effective dimension}, measures, in a sense, the number of principal directions over which the projection of the data in the RKHS is spread. They show that $\tilde{d}$ is at least as good as $\gamma_T$, precisely $\gamma_T \ge \Omega(\tilde{d}\ln \ln T)$ and thus the regret bound of KernelUCB is roughly $\tilde{O}(\sqrt{T\gamma_T})$, which is $\sqrt{\gamma_T}$ factor better than IGP-UCB. However, KernelUCB requires the number of actions to be \textit{finite}, so the regret bound is not applicable for infinite or continuum action spaces. 
 
\subsection{Regret Bound of GP-TS}
\label{subsubsec:regret-TS}
For technical reasons, we will analyze the following version of GP-TS. At each round $t$, the decision set used by GP-TS is restricted to be a {unique} discretization $D_t$ of $D$ with the property that $\abs{f(x)-f([x]_t)} \le 1/t^2$ for all $x \in D$, where $[x]_t$ is the closest point to $x$ in $D_t$. This can always be achieved by choosing a compact and convex domain $D \subset [0,r]^d$ and discretization $D_t$ with size $\abs{D_t}=(BLrdt^2)^d$ such that $\norm{x-[x]_t}_1 \le rd/BLrdt^2 = 1/BLt^2$ for all $x \in D$, where $L= \sup\limits_{x\in D}\sup\limits_{j\in [d]}\Big(\frac{\partial^2 k(p,q)}{\partial p_j \partial q_j}\at{p=q=x}\Big)^{1/2}$. This implies, for every $x \in D$,
\beq
\abs{f(x)-f([x]_t)} \le\norm{f}_k L \norm{x-[x]_t}_1 \le 1/t^2,
\label{eqn:lipschitz}
\eeq
as any $f\in H_k(D)$ is Lipschitz continuous with constant $\norm{f}_k L$ \citep[Lemma $1$]{de2012exponential}.
\begin{mytheorem}[Regret bound for GP-TS]
\label{thm:regret-bound-TS}
Let $\delta\in(0,1)$, $D \subset [0,r]^d$ be compact and convex, $\norm{f}_k \le B$ and $\{\epsilon_t\}_t$ a conditionally $R$-sub-Gaussian sequence. Running GP-TS for a function $f$ lying in the RKHS $H_k(D)$ and with decision sets $D_t$ chosen as above, with probability at least $1-\delta$, the regret of GP-TS satisfies $R_T=O\Big(\sqrt{(\gamma_T+\ln(2/\delta))d\ln (BdT)} \Big(\sqrt{T\gamma_T}+B\sqrt{T\ln(2/\delta)}\Big)\Big)$. 
\end{mytheorem}

%we obtain a regret bound of $O\big(\sqrt{Td\ln (BdT)}(B\sqrt{\gamma_T}+\gamma_T) \big)$ with high probability. More precisely, 

\textbf{Comparison with IGP-UCB.}
Observe that regret scaling of GP-TS is $\tilde{O}(\gamma_T\sqrt{dT})$ which is a multiplicative $\sqrt{d}$ factor away from the bound $\tilde{O}(\gamma_T\sqrt{T})$ obtained for IGP-UCB and similar behavior is reflected in our simulations on synthetic data. The additional multiplicative factor of $\sqrt{d\ln(BdT)}$ in the regret bound of GP-TS is essentially a consequence of discretization. How to remove this extra logarithmic dependency, and make the analysis discretization-independent, remains an open question.

\textit{\textbf{Remark.}} The regret bound for GP-TS is inferior compared to IGP-UCB in terms of the dependency on dimension $d$, but to the best of our knowledge, Theorem \ref{thm:regret-bound-TS} is the first (frequentist) regret guarantee of Thompson Sampling in the {agnostic, non-parametric setting of infinite action spaces}.

%\textbf{Bound on Maximum Information Gain.} For any compact and convex subset $D$ of $\Real^d$, $\gamma_T=O((\ln T)^{d+1})$ for Squared Exponential kernel and $\gamma_T=O(T^{d(d+1)/(2\nu+d(d+1))}\ln T)$ for Mat$\acute{e}$rn kernel (see Theorem $2$, \citet{srinivas2009gaussian}). This result, along with Theorems \ref{thm:regret-bound-UCB} and \ref{thm:regret-bound-TS}, guarantees sublinear regret of IGP-UCB and GP-TS for both kernels.

\textbf{Linear Models and a Matching Lower Bound.}
If the mean rewards are perfectly linear, i.e. if there exists a $\theta \in \Real^d$ such that $f(x)=\theta^Tx$ for all $x\in D$, then we are in the parametric setup, and one way of casting this in the kernelized framework is by using the {\em linear kernel} $k(x,x')=x^Tx'$.
%and all the methods discussed in this paper fall back to their linear counterparts: KernelUCB to LinUCB \cite{chu2011contextual}, GP-UCB to ConfidenceBall2 \cite{dani2008stochastic}, IGP-UCB to OFUL \cite{abbasi2011improved} and GP-TS to Thompson Sampling for Contextual Bandits \cite{agrawal2013thompson}.    
For this kernel, $\gamma_T =O(d\ln T)$, and the regret scaling of IGP-UCB is $\tilde{O}(d\sqrt{T})$ and that of GP-TS is $\tilde{O}(d^{3/2}\sqrt{T})$, which recovers the regret bounds of their linear, parametric  analogues OFUL \citep{abbasi2011improved} and Linear Thompson sampling \citep{agrawal2013thompson}, respectively. Moreover, in this case $\tilde{d} = d$, thus the regret of IGP-UCB is $\sqrt{d}$ factor away from that of KernelUCB. But the regret bound of KernelUCB also depends on the number of arms $N$, and if $N$ is exponential in $d$, then it also suffers $\tilde{O}(d\sqrt{T})$ regret. 
We remark that a similar $O(\ln^{3/2}T)$ factor improvement, as obtained by IGP-UCB over GP-UCB, was achieved in the linear parametric setting by \cite{abbasi2011improved} in the OFUL algorithm, over its predecessor ConfidenceBall \citep{dani2008stochastic}. Finally we see that the for linear bandit problem with infinitely many actions, IGP-UCB attains the information theoretic lower bound of $\Omega(d\sqrt{T})$ (see \cite{dani2008stochastic}), but GP-TS is a factor of $\sqrt{d}$ away from it.
% and the same observation was made by  \citet{agrawal2013thompson} also.

% Also, \citet{valko2013finite} made an incorrect statement that regret bound of \textbf{GP-UCB} depends quadratically on $\norm{f}_k$, while the dependence in \textbf{KernelUCB} is linear. But Theorem 3 of \citet{srinivas2009gaussian} shows the dependence is indeed linear, as in the case of \textbf{Improved-GP-UCB}.

%\textbf{Knowledge of Time Horizon.}
%Our algorithms doesn't require the knowledge of time horizon $T$. But if $T$ is known in advance, we can set $v_t$ and $\beta_t$ fixed in advance with $\gamma_{t-1}$ replaced by $\gamma_T$ for all $t$ and the analysis can be applied as it is to obtain the same regret upper bounds as in Theorem \ref{thm:regret-bound-UCB} and \ref{thm:regret-bound-TS}.

%\textbf{Computational Efficiency in terms of Regret Performance.}
%For linear model, \citet{agrawal2013thompson} argued that the suboptimal regret performance of Thompson sampling based algorithm comes in terms of a gain in computational efficiency. At every step, \citet{abbasi2011improved} needs to solve an NP-hard problem, even when the decision region $D$ is a polytope of $d$-dimensions. In contrast, \citet{agrawal2013thompson} runs an algorithm in time polynomial in $d$, as long as a linear function over the decision region $D$ can be efficiently optimized (For example, this can be done where $D$ is a convex set). \textcolor{red}{Can we argue in a similar manner for kernelized case also?}\\

\section{Overview of Techniques}
\label{sec:Key-Techniques}
We briefly outline here the key arguments for all the theorems in Section \ref{sec:main-results}. Formal proofs and auxiliary lemmas required are given in the appendix. 

%\begin{mylemma}
%\label{lem:self-normalized-bound}
%for any $0<\delta <1$,
%\footnotesize
%\beqan
%&&\mathbb{P}\Big[\forall t \ge 0,\norm{\epsilon_{1:t}}_{(K_t^{-1} + I)^{-1}}^2 
%\\&&\le 2R^2\ln\Big(\det(I+K_t)^{1/2}/\delta\Big)\Big] \ge 1-\delta.
%\eeqan
%\normalsize
%\end{mylemma}
\textbf{Proof Sketch for Theorem \ref{thm:self-normalized-bound}.} It is convenient to assume that $K_t$, the induced kernel matrix at time $t$, is invertible, since this is where the crux of the argument lies. 
First we show that for any function $g:D \ra \Real$ and for all $t \ge 0$, thanks to the sub-Gaussian property (\ref{eqn:noise}), the process $\left\{M_t^g \bydef \exp(\epsilon_{1:t}^Tg_{1:t}-\frac{1}{2}\norm{g_{1:t}}^2)\right\}_t$ is a non-negative super-martingale with respect to the filtration $\cF_t$, where $g_{1:t} \bydef [g(x_1),\ldots,g(x_t)]^T$ and in fact satisfies $\expect{M_t^g}\le 1$. The chief difficulty is to handle the behavior of $M_t$ at a (random) stopping time, since the sizes of quantities such as $\epsilon_{1:t}$ at the stopping time will be random. 
%First, let $\tau$ be a stopping time with respect to the filtration $\lbrace \cF_t\rbrace_{t\ge 0}$ and by convergence theorem of non-negative super-martingales, we argue that $M_\tau^g$ is well defined (whether $\tau < \infty$ or not). Next introducing a stopped version $Q_t^g = M_{\min\lbrace \tau,t \rbrace}^{g}$ of $\lbrace M_t^g\rbrace_{t\ge 0}$ and using Fatou's lemma, we obtain $\expect{M_\tau^g} \le 1$.

%Now, let $h$ be a random sample from  independent of the tail $\sigma$-algebra $\cF_\infty $ and define $M_t = \expect{M_t^h\given \cF_{\infty}} =\expect{\exp\Big(\epsilon_{1:t}^Th_{1:t}-\frac{1}{2}\norm{h_{1:t}}^2\Big)\given \cF_{\infty}}$, where the expectation (``mixing") is over $GP_D(0,k)$. 
We next construct a mixture martingale $M_t$ by mixing $M_t^g$ over $g$ drawn from an independent $GP_D(0, k)$ Gaussian process, which is a measure over a large space of functions, i.e., the space $\Real^D$. Then, by a change of measure argument, we show that this induces a mixture distribution which is essentially $\cN(0,K_t)$ over {\em any} desired finite dimension $t$, thus obtaining $M_t = \frac{1}{\sqrt{\det(I+K_t)}}\exp\Big(\frac{1}{2}\norm{\epsilon_{1:t}}_{(I+K_t^{-1})^{-1}}^2\Big)$. Next from the fact that $\expect{M_\tau}\le 1$ and from Markov's inequality, for any $\delta \in (0,1)$, we obtain
\beqn
\prob{\norm{\epsilon_{1:\tau}}_{(K_\tau^{-1} + I)^{-1}}^2 > 2\ln\Big(\sqrt{\det(I+K_\tau)}/\delta\Big)} \le \delta.
\eeqn
Finally, we lift this bound simultaneously for all $t$ through a standard stopping time construction as in \citet{abbasi2011improved}. 
%Now defining $k(x.x')=\phi(x)^T\phi(x')$ and  $\Phi_t=\big[\phi(x_1)^T,\hdots,\phi(x_t)^T\big]^T$, we observe that $S_t = \Phi_t^T\epsilon_{1:t}$, $K_t = \Phi_t\Phi_t^T$ and $V_t=I+\Phi_t^T\Phi_t$. Then we complete the proof by showing $\norm{\eps_{1:t}}_{(K_t^{-1}+I)^{-1}}= \norm{S_t}_{V_t^{-1}}$ and $\det(I+K_t)=\det(V_t)$.

%\textbf{Self-Normalized Bound for
%Vector-valued Martingales.}
%Observe that $\norm{\eps_{1:t}}_{(K_t^{-1}+I)^{-1}} = \norm{S_t}_{(\Phi_t^T\Phi_t+I)^{-1}}$, where $\Phi_t \bydef\big[k(x_1,\cdot)^T,\hdots,k(x_t,\cdot)^T\big]^T$ and $S_t \bydef \Phi_t^T\epsilon_{1:t} = \sum\limits_{s=1}^{t}\epsilon_sk(x_s,\cdot)$. Clearly $S_t$ is $\cF_{t-1}^{'}$ measurable and due to conditional sub-gaussianity $\expect{S_t\given \cF_{t-1}^{'}} = S_{t-1}$, making the sequence $\lbrace S_t \rbrace_{t\ge 0}$  a martingale (vector-valued, possibly infinite-dimensional) with respect to $\lbrace\cF^{'}_t\rbrace_{t\ge 0}$. Thus Lemma \ref{lem:self-normalized-bound} is in fact a generalization of Theorem 1 of \citet{abbasi2011improved} to the non-parametric (kernelized) case.
%\begin{mylemma}
%\label{lem:true-function-bound}
%For any $\delta \in (0,1)$,
%\footnotesize
%\beqan
%&&\mathbb{P}\Big[\forall t \ge 1, \forall x \in D, \abs{\mu_{t-1}(x)-f(x)}\\&&\le \Big(B + R\sqrt{2(\gamma_{t-1}+ \ln(1/\delta))}\Big)\sigma_{t-1}(x)\Big]\ge 1- \delta.
%\eeqan
%\normalsize
%\end{mylemma}
%\footnote{We sketch the special case of $\eta=0$, i.e. $\lambda =1$ here}

\textbf{Proof Sketch for Theorem \ref{thm:true-function-bound}.} Here we sketch the special case of $\eta=0$, i.e. $\lambda =1$. Observe that $\abs{\mu_t(x)-f(x)}$ is upper bounded by sum of two terms, $P \bydef \abs{k_t(x)^T(K_t+ I)^{-1}\eps_{1:t}}$ and $Q \bydef \abs{k_t(x)^T(K_t+I)^{-1}f_{1:t} - f(x)}$. Now we observe that $\sigma_t^2(x)=\phi(x)^T(\Phi_t^T\Phi_t + I)^{-1}\phi(x)$ and use this observation to show that $P=\abs{\phi(x)^T(\Phi_t^T\Phi_t +  I)^{-1}\Phi_t^T\eps_{1:t}}$ and $Q = \abs{\phi(x)^T(\Phi_t^T\Phi_t+ I)^{-1}f}$, which are in turn upper bounded by the terms $\sigma_t(x)\norm{S_t}_{V_t^{-1}}$ and $\norm{f}_k\sigma_t(x)$ respectively. Then the result follows using Theorem \ref{thm:self-normalized-bound}, along with the assumption that $\norm{f}_k \le B$ and the fact that $\frac{1}{2}\ln(\det(I+K_t)) \le \gamma_t$ almost surely (see Lemma \ref{lem:info-theoretic-results}) when $K_t$ is invertible. 

\textbf{Proof Sketch for Theorem \ref{thm:regret-bound-UCB}.} First from Theorem \ref{thm:true-function-bound} and the choice of $x_t$ in Algorithm \ref{algo:ucb}, we show that the instantaneous regret $r_t$ at round $t$ is upper bounded by  $2\beta_t\sigma_{t-1}(x_t)$ with probability at least $1-\delta$. Then the result follows by essentially upper bounding the term $\sum_{t=1}^{T}\sigma_{t-1}(x_t)$ by $O(\sqrt{T\gamma_T})$ (Lemma \ref{lem:bound-sum-sd} in the appendix).

\textbf{Proof Sketch for Theorem \ref{thm:regret-bound-TS}.} 
%We will analyze the following version of Algorithm \ref{algo:ts}: At each round $t$, the decision region is restricted to be a \textit{unique}\footnote{This is a mild assumption and is only to ensure the positive definiteness of the kernel matrix $K_t$ in Theorem \ref{thm:self-normalized-bound}. Note that this can be achieved, almost surely, as there are uncountably infinite number of points in $D$.} discretization $D_t \subset D$ of size $\abs{D_t}=(n_t)^d$, such that $\norm{x-[x]_t}_1 \le rd/n_t$ for all $x \in D$, where $[x]_t$ is the closest point to $x$ in $D_t$. Now, from Lemma $1$ of \citet{de2012exponential}, any $f\in H_k(D)$ is Lipschitz continuous with constant $\norm{f}_k L$, where $L= \sup\limits_{x\in D}\sup\limits_{j\in [d]}\Big(\frac{\partial^2 k(p,q)}{\partial p_j \partial q_j}\at{p=q=x}\Big)^{1/2}$. Then for all $x \in D$,
%\footnotesize
%\beq
%\abs{f(x)-f([x]_t)} \le\norm{f}_k L \norm{x-[x]_t}_1 \le \frac{BLrd}{n_t} = 1/t^2,
%\label{eqn:lipschitz}
%\eeq
%\normalsize
%where we choose $n_t = BLrdt^2$. This gives the size of discretization $D_t$ at round $t$ to be $(BLrdt^2)^d$ and
We follow a similar approach given in \citet{agrawal2013thompson} to prove the regret bound of GP-TS. First observe that from our choice of discretization sets $D_t$, the instantaneous regret at round $t$ is given by $r_t = f(x^\star)-f([x^\star]_t)+f([x^\star]_t)-f(x_t)
\le \frac{1}{t^2} + \Delta_t(x_t)$, where $\Delta_t(x)\bydef f([x^\star]_t)-f(x)$ and $[x^\star]_t$ is the closest point to $x^\star$ in $D_t$. Now
at each round $t$, after an action is chosen, our algorithm improves the confidence about true reward function $f$, via an update of $\mu_t(\cdot)$ and $k_t(\cdot,\cdot)$. However, if we play a suboptimal arm, the regret suffered can be much higher than the improvement of our knowledge. To overcome this difficulty, at any round $t$, we divide the arms (in the present discretization $D_t$) into two groups: \textit{saturated arms}, $S_t$, defined as those with $\Delta_t(x) > c_t\sigma_{t-1}(x)$ and \textit{unsaturated} otherwise, where $c_t$ is an appropriate constant
%, where $c_t \bydef v_t(1+\hat{c}_t)$ and $\hat{c}_t$ is a function of $t,d,r,B$ 
(see Definition \ref{def:constants}, \ref{def:saturated-arms}). The idea is to show that the probability of playing a saturated arm is small and then bound the regret of playing an unsaturated arm in terms of standard deviation. This is useful because the inequality $\sum_{t=1}^{T}\sigma_{t-1}(x_t)\le O(\sqrt{T\gamma_T})$ (Lemma \ref{lem:bound-sum-sd}) allows us to bound the total regret due to unsaturated arms.

First we lower bound the probability of playing an unsaturated arm at round $t$. We define a filtration $\cF^{'}_{t-1}$ as the history $\cH_{t-1}$ up to round $t-1$ and prove that for ``most" (in a high probability sense) $\cF^{'}_{t-1}$, $\prob{x_t \in D_t\setminus S_t\given \cF^{'}_{t-1}} \ge p-1/t^2$, where $p=1/4e\sqrt{\pi}$ ( Lemma \ref{lem:prob-playing-saturated-arms}). This observation, along with concentration bounds for $f_t(x)$ and $f(x)$ (Lemma \ref{lem:event-concentration}) and ``smoothness" of $f$ (Equation \ref{eqn:lipschitz}), allow us to show that the expected regret at round $t$ is upper bounded in terms of $\sigma_{t-1}(x_t)$, i.e. in terms of regret due to playing an unsaturated arm. More precisely, we show that for ``most" $\cF^{'}_{t-1}$, $\expect{r_t\given \cF^{'}_{t-1}} \le \frac{11c_t}{p}\expect{\sigma_{t-1}(x_t)\given \cF^{'}_{t-1}} + \frac{2B+1}{t^2}$ (Lemma \ref{lem:bound-on-instantaneous-regret}), and use it to prove that $X_t \simeq r_t-\frac{11c_t}{p}\sigma_{t-1}(x_t) - \frac{2B+1}{t^2};t\ge 1$ is a super-martingale difference sequence adapted to filtration $\lbrace\cF^{'}_t\rbrace_{t\ge 1}$ (Lemma \ref{lem:supermartingale}). Now, using the Azuma-Hoeffding inequality (Lemma \ref{lem:total-regret}), along with the bound on $\sum_{t=1}^{T}\sigma_{t-1}(x_t)$, we obtain the desired high-probability regret bound.

%\textbf{\textit{Remark.}}
%Our choice of discretization along with ``smoothness" of $f$ (Equation \ref{eqn:lipschitz}) ensure that an additive constant factor will be accumulated in the final bound.
\section{Experiments}
\label{sec:Experiments}
In this section we provide numerical results on both synthetically generated test functions and functions from real-world data. We compare GP-UCB, IGP-UCB and GP-TS with GP-EI and GP-PI\footnote{GP-EI and PI perform similarly and thus are not separately distinguishable in the plots.}. 
%All the experiments are performed on finite decision sets as prevalent in the literature \cite{hoffman2011portfolio}.

{\bf Synthetic Test Functions.} %
\label{subsec:synthetic}
%As the RKHS for a specific kernel is literally constructed by completing the set of posterior mean functions, 
We use the following procedure to generate test functions from the RKHS. First we sample $100$ points uniformly from the interval $[0,1]$ and use that as our decision set. Then we compute a kernel matrix $K$ on those points and draw reward vector $y \sim \cN(0,K)$. Finally, the mean of the resulting posterior distribution is used as the test function $f$. We set noise parameter $R^2$ to be $1\%$ of function range and use $\lambda = R^2$. We used Squared Exponential kernel with lengthscale parameter $l = 0.2$ and Mat$\acute{e}$rn kernel with parameters $\nu=2.5,l=0.2$. Parameters $\beta_t,\tilde{\beta}_t,v_t$ of IGP-UCB, GP-UCB and GP-TS are chosen as given in Section \ref{sec:Algorithms}, with $\delta=0.1, B^2= f^TKf$ and $\gamma_t$ set according to theoretical upper bounds for corresponding kernels. We run each algorithm for $T=30000$ iterations, over $25$ independent trials (samples from the RKHS) and plot the average cumulative regret along with standard deviations (Figure \ref{fig:synthetic_plot_rkhs}). We see a significant improvement in the performance of IGP-UCB over GP-UCB. In fact IGP-UCB performs the best in the pool of competitors, while GP-TS also fares reasonably well compared to GP-UCB and GP-EI/GP-PI.

We next sample $25$ random functions from the $GP(0,K)$ and perform the same experiment (Figure \ref{fig:synthetic_plot_gp}) for both kernels with exactly same set of parameters. The relative performance of all methods is similar to that in the previous experiment, which is the arguably harder ``agnostic'' setting of a fixed, unknown target function. 
\begin{figure}[t!]
% \vskip -5mm
\centering
\subfigure[]{\includegraphics[height=2.5in,width=3in]{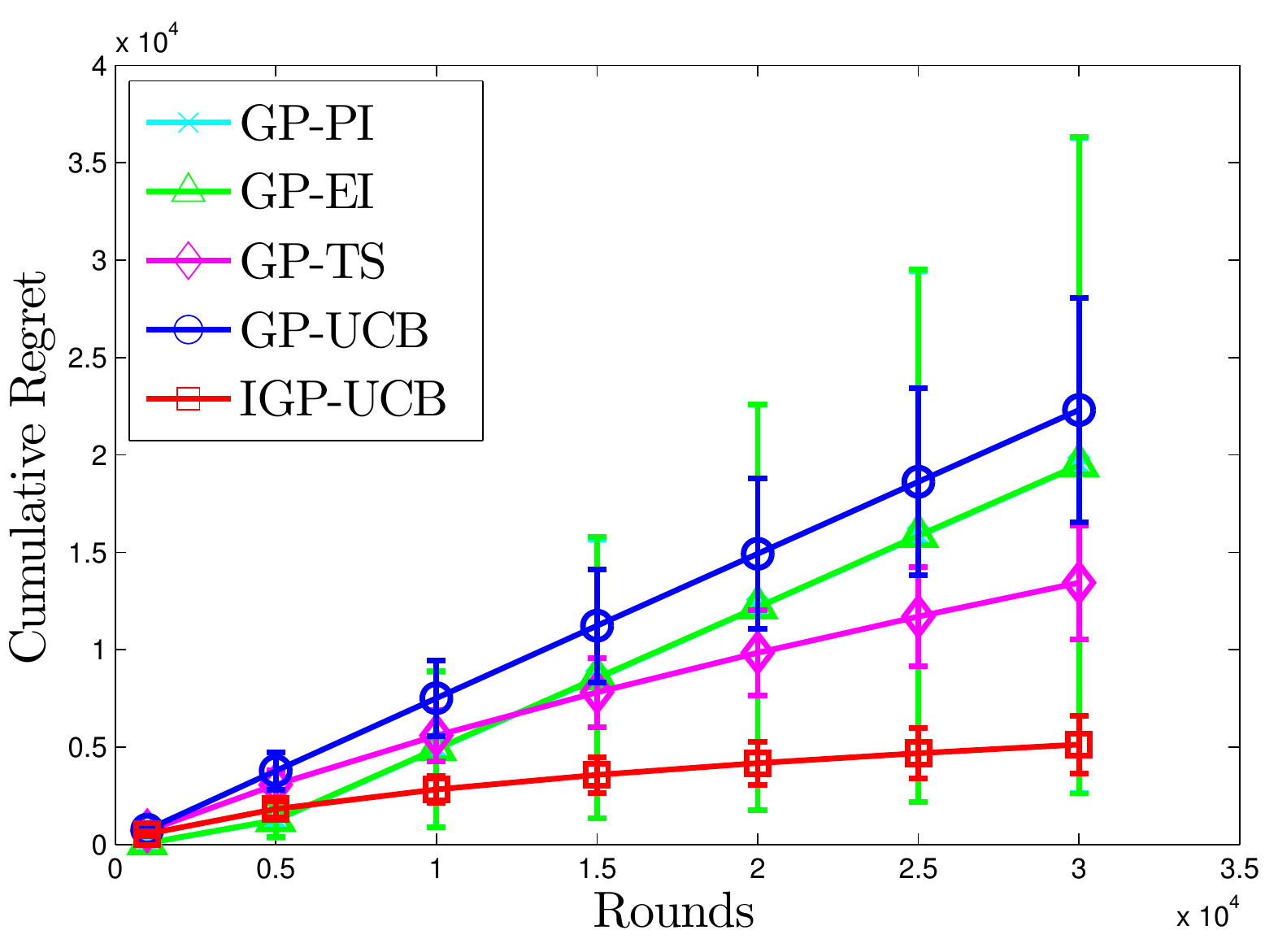}}\subfigure[]{\includegraphics[height=2.5in,width=3in]{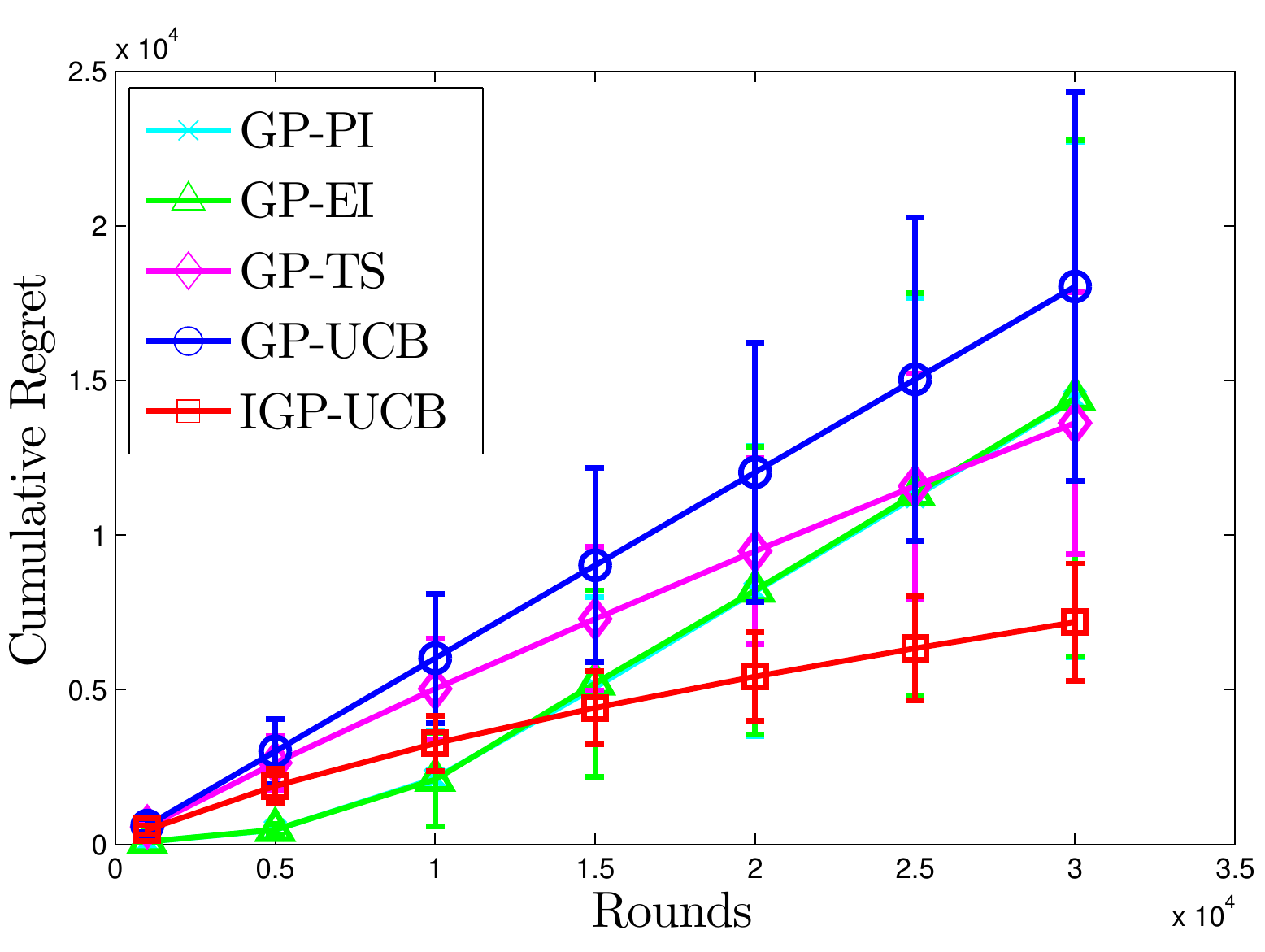}}
\caption{Cumulative regret for functions lying in the RKHS corresponding to (a) Squared Exponential kernel and (b) Mat$\acute{e}$rn kernel.% IGP-UCB improves upon GP-UCB and performs best among %the lot. GP-TS performs better than GP-UCB and GP-EI/GP-PI.
}
\label{fig:synthetic_plot_rkhs} 
\vskip -5mm
\end{figure}
\begin{figure}[t!]
\centering
\subfigure[]{\includegraphics[height=2.5in,width=3in]{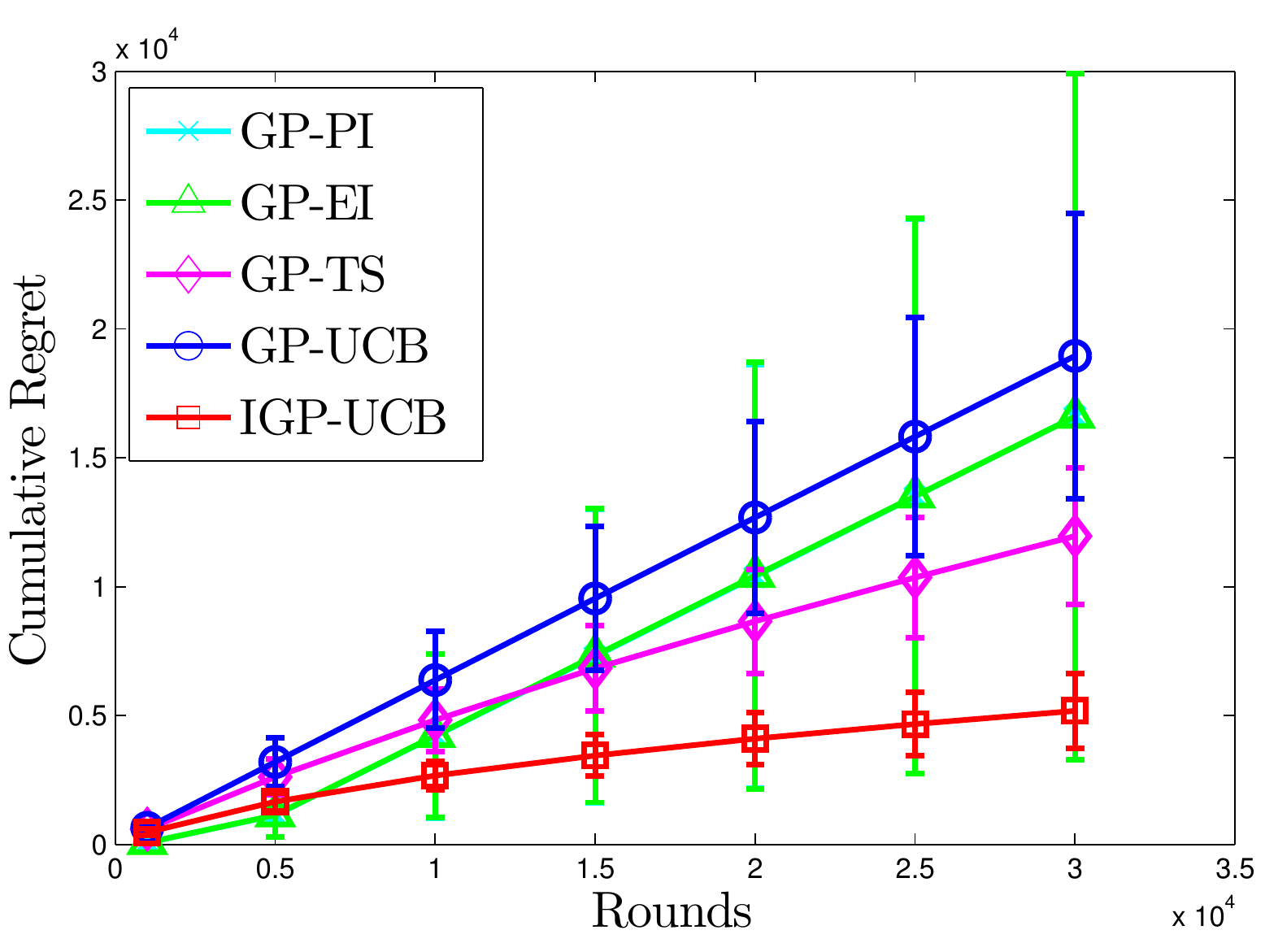}}\subfigure[]{\includegraphics[height=2.5in,width=3in]{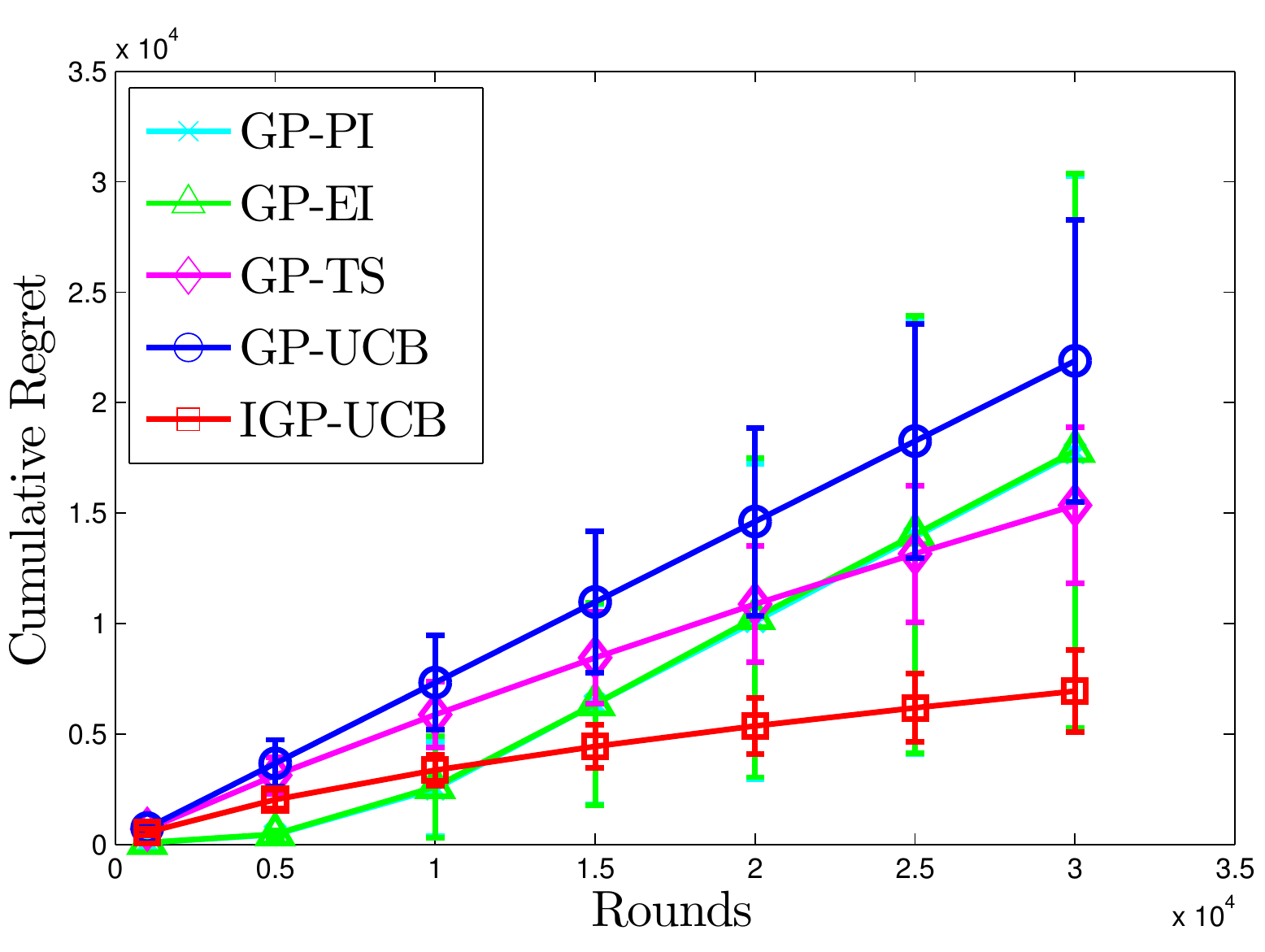}}
\caption{Cumulative regret for functions lying in the GP corresponding to (a) Squared Exponential kernel and (b) Mat$\acute{e}$rn kernel. 
% IGP-UCB and GP-TS perform much better than GP-UCB/EI/PI.
}
\label{fig:synthetic_plot_gp} 
\vskip -5mm
\end{figure}

{\bf Standard Test Functions.} %
\label{subsec:standard}
We consider $2$ well-known synthetic benchmark functions for Bayesian Optimization: \textit{Rosenbrock} and \textit{Hartman}3 (see \citet{azimi2012hybrid} for exact analytical expressions). We sample $100\;d$ points uniformly from the domain of each benchmark function, $d$ being the dimension of respective domain, as the decision set. We consider the Squared Exponential kernel with $l=0.2$ and set all parameters exactly as in previous experiment. The cumulative regret for $25$ independent trials on \textit{Rosenbrock} and \textit{Hartman3} benchmarks is shown in Figure \ref{fig:standard_plot}. We see GP-EI/PI perform better than the rest, while IGP-UCB and GP-TS show competitive performance. Here no algorithm is aware of the underlying kernel function, hence we conjecture that the UCB- and TS- based algorithms are somewhat less robust on the choice of kernel than EI/PI. 

\begin{figure}[t!]
\centering
\subfigure[]{\includegraphics[height=2.5in,width=3in]{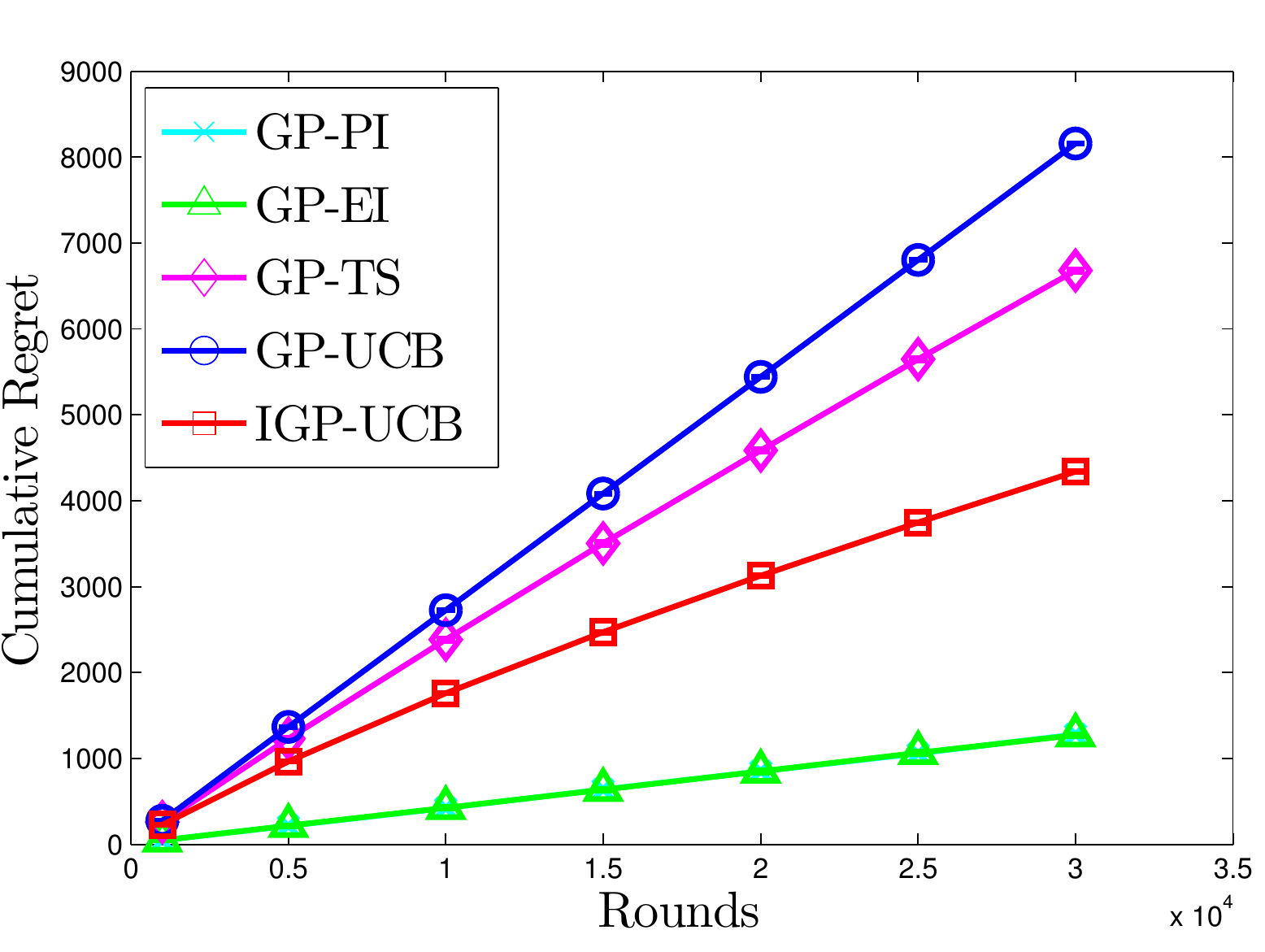}}\subfigure[]{\includegraphics[height=2.5in,width=3in]{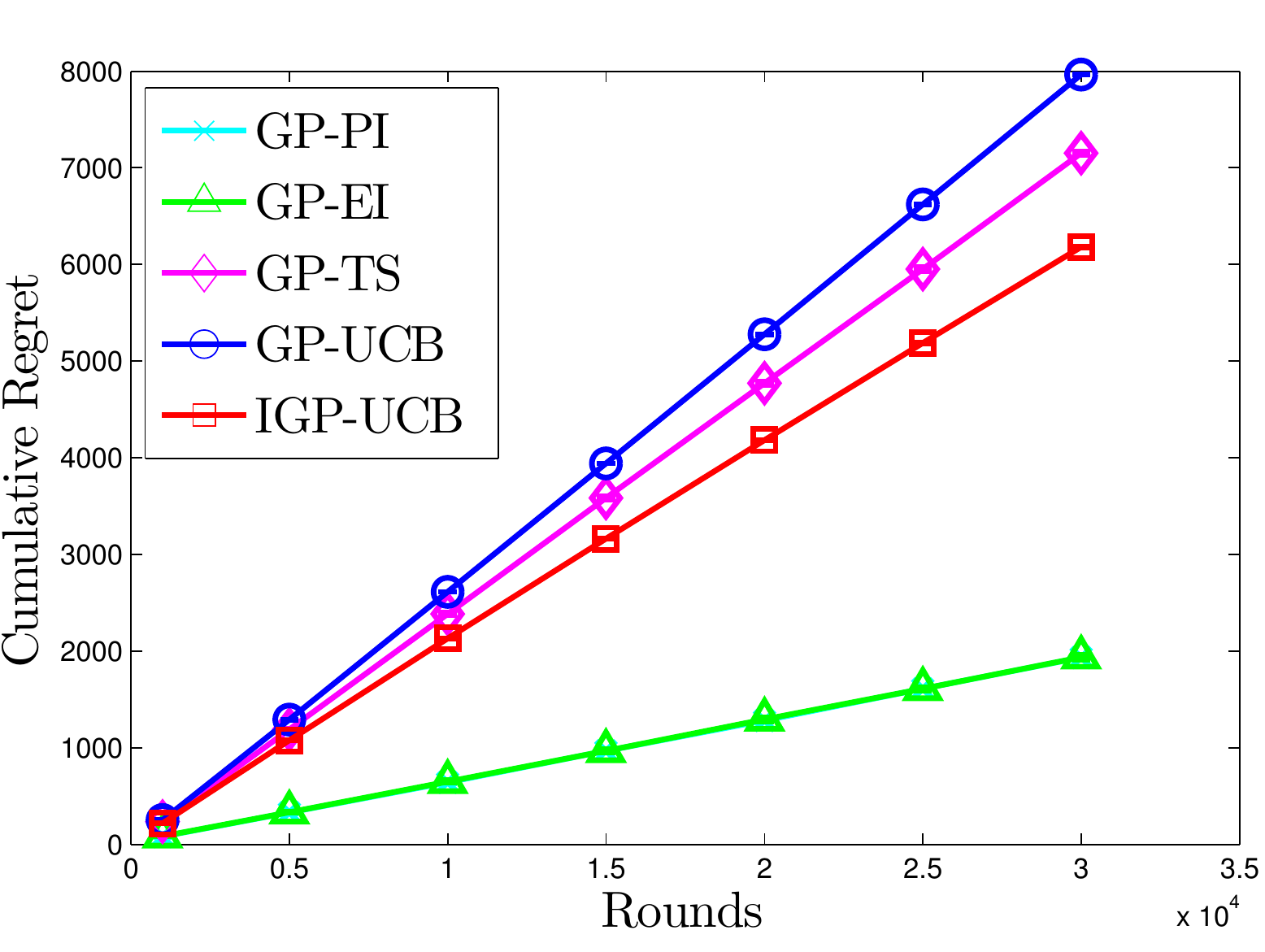}} 
\caption{Cumulative regret for (a) \textit{Rosenbrock} and (b) \textit{Hartman3} benchmark function. 
%Improvement is achieved in IGP-UCB over GP-UCB and GP-TS performs comparably with IGP-UCB while GP-EI/PI are better than others.
}
\label{fig:standard_plot} 
\vskip -5mm
\end{figure}

\textbf{Temperature Sensor Data.} %
\label{subsec:temp-sensor}
We use temperature data\footnote{\url{http://db.csail.mit.edu/labdata/labdata.html}} collected from 54 sensors deployed in the Intel Berkeley Research lab between February 28th and April 5th, 2004 with samples collected at 30 second intervals. We
tested all algorithms in the context of learning the maximum reading of the sensors collected between 8 am to 9 am. We take measurements of first 5 consecutive days
(starting Feb. 28th 2004) to learn algorithm parameters. Following \citet{srinivas2009gaussian}, we calculate the empirical covariance matrix of the sensor measurements and use it as the kernel matrix in the algorithms. Here $R^2$ is set to be $5\%$ of the average empirical variance of sensor readings and other algorithm parameters is set similarly as in the previous experiment with $\gamma_t=1$ (found via cross-validation). The functions for testing consist of one set of measurements from all sensors in the two following days and the cumulative regret is plotted over all such test functions. From Figure \ref{fig:sensor_plot}, we see that IGP-UCB and GP-UCB performs the same, while GP-TS outperforms all its competitors.
\begin{figure}[t!]
\centering
\subfigure[]{\includegraphics[height=2.5in,width=3in]{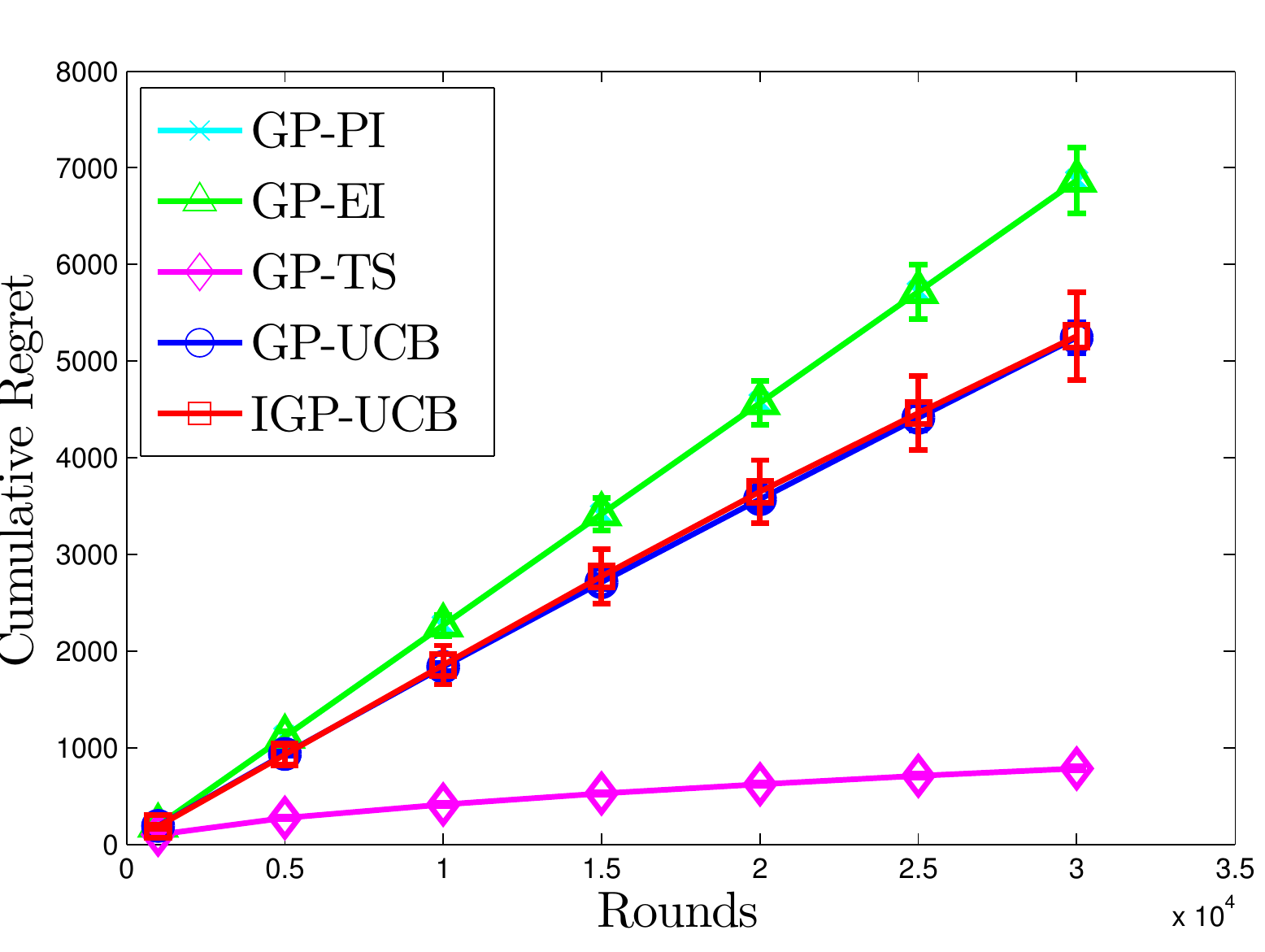}}\subfigure[]{\includegraphics[height=2.5in,width=3in]{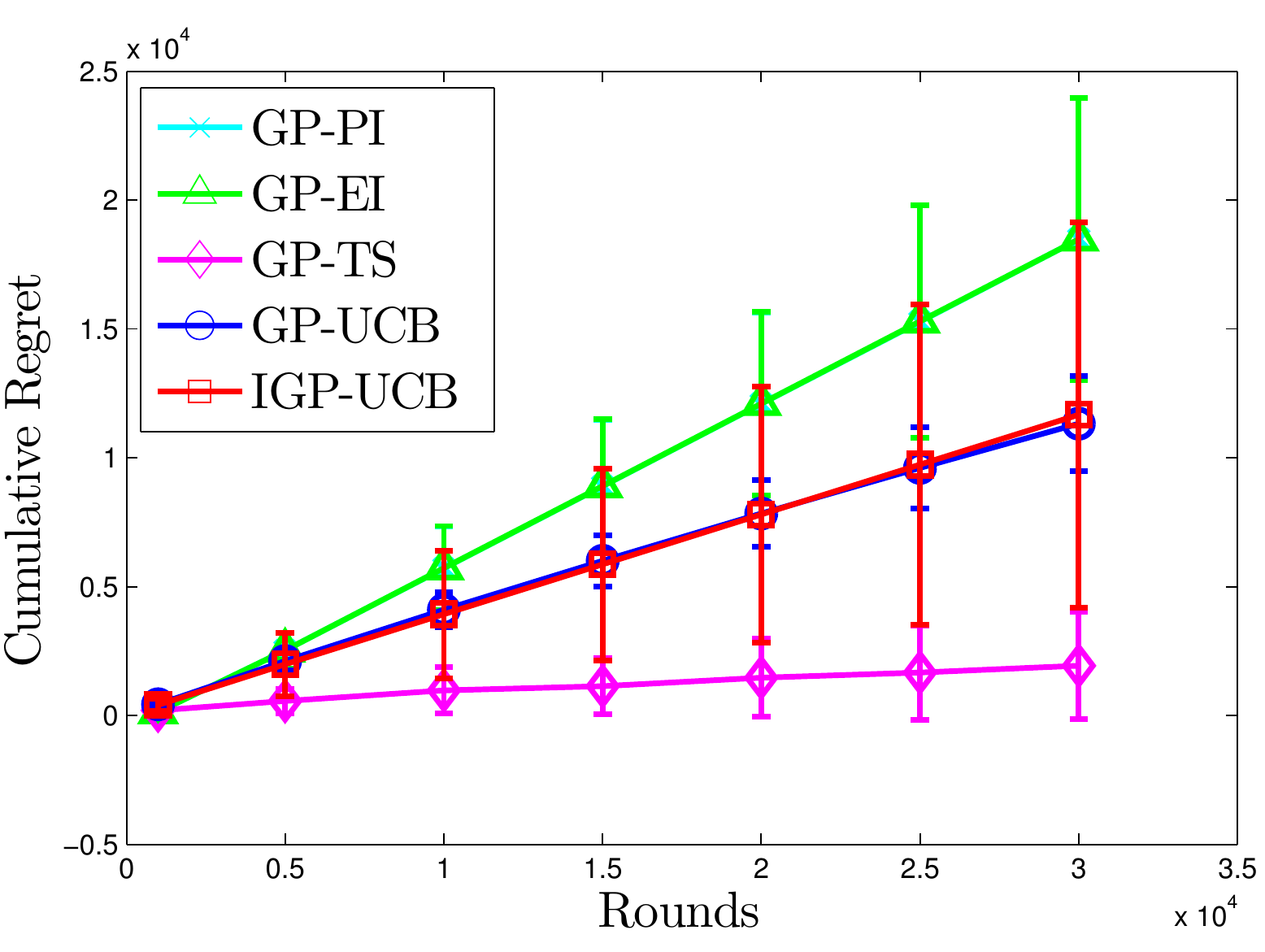}}
\caption{Cumulative regret plots for (a) temperature data and (b) light sensor data. 
% GP-UCB and IGP-UCB show similar performance, where GP-TS outperforms all competitors.
}
\label{fig:sensor_plot} 
\vskip -5mm
\end{figure}

\textbf{Light Sensor Data.} %
\label{subsec:light-sensor}
We take light sensor data collected in the CMU Intelligent Workplace in Nov 2005, which is available online as Matlab
structure\footnote{\url{http://www.cs.cmu.edu/~guestrin/Class/10708-F08/projects/lightsensor.zip}} and contains locations of $41$ sensors, $601$ train samples and $192$ test samples. We compute the kernel matrix, estimate the noise and set other algorithm parameters exactly as in the previous experiment. Here also GP-TS is found to perform better than the others, with IGP-UCB performing better than GP-EI/PI (Figure \ref{fig:sensor_plot}).

{\bf Related work.} An alternative line of work pertaining to $\mathcal{X}$-armed bandits \cite{kleinberg2008multi,bubeck2011x,carpentier2015simple,azar2014online} studies continuum-armed bandits with smoothness structure. %
%A key insight is that when the target value function is sufficiently smooth (a natural property in many physical problem settings), one can effectively learn about the values of (infinitely many) neighboring arms by playing a particular arm, making it possible to attain regret that is sublinear in the number of rounds of optimization \cite{refs}. 
For instance, \cite{bubeck2011x} show that with a Lipschitzness assumption on the reward function, algorithms based on discretizing the domain yield nontrivial regret guarantees, of order $\Omega(T^{\frac{d+1}{d+2}})$ in $\Real^d$. %
%Another relevant area of work models the reward function as a smooth graph function \cite{valko2014spectral,kocak2014spectral}.
%Bayesian approaches have received success in a broad range of optimization problems and particularly GP-based models provide a powerful approach in this regard (see \citet{brochu2010tutorial}, \citet{lizotte2008practical}).
Other Bayesian approaches to function optimization are GP-EI \cite{movckus1975bayesian}, GP-PI \cite{kushner1964new}, GP-EST \cite{wang2016optimization} and GP-UCB, including the contextual \cite{krause2011contextual}, high-dimensional \cite{djolonga2013high,wang2013bayesian}, time-varying \cite{bogunovic2016time} safety-aware \cite{gotovos2015safe}, budget-constraint \cite{hoffman2013exploiting} and noise-free \cite{de2012exponential} settings. Other relevant work focuses on best arm identification problem in the Bayesian setup considering pure exploration \cite{grunewalder2010regret}. %
%While most of these works builds upon the UCB principle to solve the GP optimization problem, not much attention is given towards Thompson Sampling \cite{thompson1933likelihood,agrawal2012analysis,kaufmann2012thompson,gopalan2014thompson} framework. 
For Thompson sampling (TS), \citet{russo2014learning} analyze the Bayesian regret of TS, which includes the case where the target function is sampled from a GP prior. Our work obtains the first frequentist regret of TS for unknown, fixed functions from an RKHS.
\vskip -5mm

%\cite{srinivas2012information} .
\section{Conclusion}
\label{sec:Conclusion}

For bandit optimization, we have improved upon the existing GP-UCB algorithm, and introduced a new GP-TS algorithm. The proposed algorithms perform well in practice both on synthetic and real-world data. An interesting case is when the kernel function is also not known to the algorithms a priori and needs to be learnt adaptively. Moreover, one can consider classes of time varying functions from the RKHS, and general reinforcement learning with GP techniques. There are also important questions on computational aspects of optimizing functions drawn from GPs.

%\newpage
\bibliography{paper.bib}
%\bibliographystyle{icml2017}
%\newpage
%\onecolumn
\section*{Appendix}
\section*{A. Proof of Theorem \ref{thm:self-normalized-bound}}
For a function $ g:D \ra \Real$ and a sequence of reals $n \equiv (n_t)_{t=1}^\infty$, define for any $t \ge 0 $
\beqn
M_t^{g,n} = \exp\Big(\epsilon_{1:t}^Tg_{1:t,n}-\frac{R^2}{2}\norm{g_{1:t,n}}^2\Big),
\eeqn
where the vector $g_{1:t,n} := [g(x_1) + n_1,\ldots,g(x_t) + n_t]^T$. 
%Consider the $\sigma$-algebra $\cF_t = \sigma \Big(\lbrace x_s, \epsilon_s\rbrace_{s = 1}^{t},x_{t+1}\Big)$.
We first establish the following technical result, which resembles
\citet[Lemma 8]{abbasi2011improved}. 
\begin{mylemma}
For fixed $g$ and $n$, $\lbrace M_t^{g,n} \rbrace_{t=0}^{\infty}$ is a super-martingale with respect to the filtration $\lbrace \cF_t\rbrace_{t=0}^{\infty}$.
\end{mylemma}

\begin{proof}
First, define
\beqn
\Delta_t^{g,n} := \exp\Big(\epsilon_t(g(x_t) + n_t)-\frac{R^2}{2}(g(x_t) + n_t)^2\Big).
\eeqn
Since $x_t$ is $\cF_{t-1}$-measurable and $\epsilon_t$ is $\cF_t$-measurable, $M_t^{g,n}$ as well as $\Delta_t^{g,n}$ are $\cF_t$ measurable. Also, by the conditional $R$-sub-Gaussianity of $\epsilon_t$, we have
\beqn
\forall \lambda \in \Real, \;\; \expect{e^{\lambda \epsilon_t} \given \cF_{t-1}} \le \exp\left(\frac{\lambda^2R^2}{2}\right),
\eeqn
which in turn implies $\expect{\Delta_t^{g,n}\given \cF_{t-1}} \le 1$. We also have
\beqan
&&\expect{M_t^{g,n}\given \cF_{t-1}}\\
&=&  \expect{M_{t-1}^{g,n} \Delta_t^{g,n}\given \cF_{t-1}} = M_{t-1}^{g,n} \expect{\Delta_t^{g,n} \given \cF_{t-1}} \le M_{t-1}^{g,n},
\eeqan
showing that $\lbrace M_t^{g,n} \rbrace_{t=0}^{\infty}$ is a non-negative super-martingale and proving the lemma.  
%and that, in fact,
%
%\beqn
%\expect{M_t^g} \le \expect{M_{t-1}^g} \le \cdots \le \expect{M_0^g} = \expect{1} = 1 \;\; \forall t.
%\eeqn
%
\end{proof}
Also observe that $\expect{M_t^{g,n}} \le 1$ for all $t$, as
\beqn
\expect{M_t^{g,n}} \le \expect{M_{t-1}^{g,n}} \le \cdots \le \expect{M_0^{g,n}} = \expect{1} = 1.
\eeqn
Now, let $\tau$ be a stopping time with respect to the filtration $\lbrace \cF_t\rbrace_{t=0}^{\infty}$. By the convergence theorem for nonnegative super-martingales \citep{dur05:probtebook}, $M_{\infty}^{g,n} = \lim\limits_{t\ra \infty}M_t^{g,n}$ exists almost surely, and thus $M_\tau^{g,n}$ is well-defined. Now let $Q_t^{g,n} = M_{\min\lbrace \tau,t \rbrace}^{g,n}$, $t \ge 0$, be a stopped version of $\lbrace M_t^{g,n}\rbrace_t$. By Fatou's lemma \citep{dur05:probtebook},
\beqa
\expect{M_\tau^{g,n}} &=& \expect{\lim_{t\ra \infty}Q_t^{g,n}} = \expect{\liminf_{t\ra \infty}Q_t^{g,n}}\nonumber\\
&\le & \liminf_{t \ra \infty} \expect{Q_t^{g,n}}\nonumber\\ &=& \liminf_{t \ra \infty} \expect{M_{\min\lbrace \tau,t \rbrace}^{g,n}} \le 1,
\label{eqn:Fatou}
\eeqa
since the stopped super-martingale $\left(M_{\min\lbrace \tau,t \rbrace}^{g,n}\right)_t$ is also a super-martingale \citep{dur05:probtebook}. 

Now, let $\cF_{\infty}$ be the $\sigma$-algebra generated by $\lbrace\cF_t\rbrace_{t=0}^{\infty}$, and let
$N \equiv (N_t)_{t =1}^\infty$ be a sequence of independent and identically distributed Gaussian random variables with mean $0$ and variance $\eta$, independent of $\cF_\infty$. Let $h: D \to \Real$ be a random function distributed according to the Gaussian process measure $GP_D(0,k)$, and independent of both $\cF_{\infty}$ and $(N_t)_{t =1}^\infty$.

 % Clearly $\expect{M_t^h\given h} \le 1$ for all $t$. 
 
For each $t \ge 0$, define $M_t = \expect{M_t^{h,N}\given \cF_{\infty}}$. In words, $(M_t)_t$ is a mixture of super-martingales of the form $M_t^{g,n}$, and it is not hard to see that $(M_t)_t$ is also a (non-negative) super-martingale w.r.t. the filtration $\{\cF_t\}_t$, hence $M_{\infty}=\lim\limits_{t\ra \infty}M_t$ is well-defined almost surely. We can write
\beqn
\expect{M_t} = \expect{M_t^{h,N}} = \expect{\expect{M_t^{h,N}\given h, N}} \le \expect{1} = 1 \;\; \forall t.
\eeqn
An argument similar to (\ref{eqn:Fatou}) also shows that $\expect{M_\tau} \le 1$ for any stopping time $\tau$. Now, without loss of generality, we assume $R=1$ (this can always be achieved through appropriate scaling), and compute
\beqan
M_t &=& \expect{\exp\Big(\epsilon_{1:t}^Th_{1:t,N}-\frac{1}{2}\norm{h_{1:t,N}}^2\Big)\given \cF_{\infty}}\\
&=& \int_{\Real^D} \int_{\Real^t}  \exp\Big(\epsilon_{1:t}^T ([h(x_1) \ldots h(x_t)]^T + z) -\frac{1}{2}\norm{[h(x_1) \ldots h(x_t)]^T + z}^2\Big) d\mu_1(h) d\mu_2(z)\\
&=& \int_{\Real^t}\exp\Big(\epsilon_{1:t}^T \lambda-\frac{1}{2}\norm{\lambda}^2\Big)f(\lambda) d\lambda, %\;\; [\star]
\eeqan
where $\mu_1$ is the Gaussian process measure $GP_D(0,k)$ over the function space $\Real^D \equiv \{g: D \to \Real \}$, $\mu_2$ is the multivariate Gaussian distribution on $\Real^t$ with mean $0$ and covariance $\eta I$ where $I$ is the identify, $du$ is standard Lebesgue measure on $\Real^t$, and $f$ is the density of the random vector $[h(x_1) \ldots h(x_t)]^T + z$, which is distributed as the multivariate Gaussian $\cN(0,K_t + \eta I)$ given the sampled points $x_1,\ldots, x_t$ up to round $t$, where $K_t$ is the induced kernel matrix at time $t$ given by $K_t(i,j) = k(x_i, x_j)$, $1 \le i, j \le t$. (Note: $K_t$ is not positive definite and invertible when there are repetitions among $(x_1, \ldots, x_t)$, but $K_t + \eta I$ is).

%
%. Clearly, our choice of  $x_1,\ldots, x_t$ ensures $K_t$ is positive-definite and hence invertible.
%\todo{Check $[\star]$ is valid when $t = \tau$ , the random stopping time and also justify this is true for $\tau=\infty$.}\\
Thus, we have
\beqan
& M_t&=\frac{1}{\sqrt{(2\pi)^t \det(K_t + \eta I)}}\int_{\Real^t}\exp\Bigg(\epsilon_{1:t}^T \lambda-\frac{\norm{\lambda}^2}{2}-\frac{\norm{\lambda}^2_{(K_t + \eta I)^{-1}}}{2}\Bigg)d\lambda\\
&=& \frac{\exp\Big(\frac{\norm{\epsilon_{1:t}}^2}{2}\Big)}{\sqrt{(2\pi)^t \det(K_t + \eta I)}}\int_{\Real^t}\exp\Bigg(-\frac{\norm{\lambda -\epsilon_{1:t}}^2}{2}-\frac{\norm{\lambda}^2_{(K_t+ \eta I)^{-1}}}{2}\Bigg)d\lambda.
\eeqan
Now for positive-definite matrices $P$ and $Q$
\beqn
\norm{x-a}_P^2 + \norm{x}_Q^2 = \norm{x-(P+Q)^{-1}Pa}_{P+Q}^2 + \norm{a}_P^2 -\norm{Pa}_{(P+Q)^{-1}}^2.
\eeqn
Therefore,
\beqan
&&\norm{\lambda -\epsilon_{1:t}}_I^2 + \norm{\lambda}^2_{(K_t+ \eta I)^{-1}}\\ &=&\norm{\lambda - (I+(K_t+ \eta I)^{-1})^{-1}I\epsilon_{1:t}}_{I+(K_t+ \eta I)^{-1}}^2 + \norm{\epsilon_{1:t}}_I^2 - \norm{I\epsilon_{1:t}}_{(I+(K_t+ \eta I)^{-1})^{-1}}^2,
\eeqan
which yields 
\beqan
M_t &=& \frac{1}{\sqrt{(2\pi)^t \det(K_t+ \eta I)}}\exp\Big(\frac{1}{2}\norm{\epsilon_{1:t}}_{(I+(K_t+ \eta I)^{-1})^{-1}}^2\Big)\\
&&\times\int_{\Real^t}\exp\Big(-\frac{1}{2}\norm{\lambda-(I+(K_t+ \eta I)^{-1})^{-1}\epsilon_{1:t}}^2_{I+(K_t+ \eta I)^{-1}}\Big)d\lambda\\
&=&\frac{1}{\sqrt{\det(K_t+ \eta I)\det((K_t+ \eta I)^{-1}+I)}} \exp\Big(\frac{1}{2}\norm{\epsilon_{1:t}}_{(I+(K_t+ \eta I)^{-1})^{-1}}^2\Big)\\
&=& \frac{1}{\sqrt{\det(I+K_t+ \eta I)}}\exp\Big(\frac{1}{2}\norm{\epsilon_{1:t}}_{(I+(K_t+ \eta I)^{-1})^{-1}}^2\Big),
\eeqan
since for any positive definite matrix $A \in \Real^t$,
\beqan 
&&\int_{\Real^t}\exp\Big(-\frac{1}{2}(x-a)^TA(x-a)\Big)dx = \int_{\Real^t}\exp\Big(-\frac{1}{2}\norm{x-a}_A^2\Big)dx =\sqrt{(2\pi)^t/\det(A)}.
\eeqan
%Now consider $\tau < \infty$. 
Now as $\expect{M_\tau} \le 1$, using Markov's inequality gives,
%using Doob's inequality \todo{How to justify this, i.e. how to show ${M_t}$ is a super-martingale} 
for any $\delta \in (0,1)$,
\beqa
&&\prob{\norm{\epsilon_{1:\tau}}_{((K_\tau + \eta I)^{-1} + I)^{-1}}^2 > 2\ln\Big(\sqrt{\det((1+\eta)I+K_\tau)}/\delta\Big)}\nonumber \\ &=&
\prob{M_\tau > 1/\delta} < \delta \expect{M_\tau} \le \delta.
\label{eqn:stopping-time-bound}
\eeqa
To complete the proof, we now employ a stopping time construction as in \citet{abbasi2011improved}. For each $t \ge 0$, define the `bad' event
\beqn
B_t(\delta) = \Big\lbrace \omega \in \Omega : \norm{\epsilon_{1:t}}_{((K_t + \eta I)^{-1} + I)^{-1}}^2 > 2\ln\Big(\sqrt{\det((1+\eta)I+K_t)}/\delta\Big) \Big\rbrace,
\eeqn
so that
\beqan 
&&\prob{\bigcup\limits_{t\ge0}B_t(\delta)}\\ &=& \prob{\exists t \ge 0:\norm{\epsilon_{1:t}}_{((K_t+ \eta I)^{-1} + I)^{-1}}^2 \le  2\ln\Big(\sqrt{\det((1+\eta)I+K_t)}/\delta\Big)},
\eeqan
which is the probability required to be bounded by $\delta$ in the statement of the theorem. 

Let $\tau'$ be the first time when the bad event $B_t(\delta)$ happens, i.e., $\tau'(\omega) := \min\lbrace t \ge 0:\omega \in B_t(\delta)\rbrace$, with $\min\lbrace\emptyset\rbrace := \infty$ by convention. Clearly, $\tau'$ is a stopping time, and
\beqn
\bigcup\limits_{t\ge0}B_t(\delta) = \lbrace \omega \in \Omega : \tau'(\omega)< \infty\rbrace.
\eeqn
Therefore, we can write
\beqan
&&\prob{\bigcup\limits_{t\ge0}B_t(\delta)}\\ &=& \prob{\tau' < \infty}\\
&=& \prob{\norm{\epsilon_{1:\tau'}}_{((K_{\tau'}+ \eta I)^{-1} + I)^{-1}}^2 > 2\ln\Big(\sqrt{\det((1+\eta)I+K_{\tau'})}/\delta\Big),\tau' < \infty}\\
&\le & \prob{\norm{\epsilon_{1:\tau'}}_{((K_{\tau'}+ \eta I)^{-1} + I)^{-1}}^2 > 2\ln\Big(\sqrt{\det((1+\eta)I+K_{\tau'})}/\delta\Big)} \le \delta,
\eeqan
by the inequality (\ref{eqn:stopping-time-bound}). 

When $K_t$ is positive definite (and hence invertible) for each $t \ge 1$, one can use a similar construction as in Part 1, with $\eta = 0$ (i.e., $N$ is the all-zeros sequence with probability 1), to recover the corresponding conclusion (\ref{eqn:thmpart1}) with $\eta = 0$. 
\qed

\section*{Proof of Lemma \ref{lem:selfnorm}}
Define, for each time $t$, the $t \times \infty$ matrix $\Phi_t := [\phi(x_1) \cdots \phi(x_t)]^T$, and observe that $V_t = I + \Phi_t^T \Phi_t$ and $K_t = \Phi_t \Phi_t^T$. With this, we can compute
\begin{align*} % requires amsmath; align* for no eq. number
   \norm{S_t}^2_{V_t^{-1}} &= S_t^T V_t^{-1} S_t = \sum_{s=1}^{t}\epsilon_s\phi(x_s)^T  \left(I + \Phi_t^T \Phi_t \right)^{-1} \sum_{s=1}^{t}\epsilon_s\phi(x_s) \\
   &= \epsilon_{1:t}^T \Phi_t \left(I + \Phi_t^T \Phi_t \right)^{-1} \Phi_t^T \epsilon_{1:t} \\
   &= \epsilon_{1:t}^T \Phi_t  \Phi_t^T \left( \Phi_t \Phi_t^T + I \right)^{-1} \epsilon_{1:t} \\
   &= \epsilon_{1:t}^T K_t \left( K_t + I \right)^{-1} \epsilon_{1:t} \\
   &= \epsilon_{1:t}^T \left( K_t^{-1} + I \right)^{-1} \epsilon_{1:t} = \norm{\epsilon_{1:t}}^2_{\left( K_t^{-1} + I \right)^{-1}},
\end{align*}
completing the proof. 

\section*{B. Information Theoretic Results}
\begin{mylemma}
For every $t \ge 0$, the maximum information gain $\gamma_t$, for the points chosen by Algorithm \ref{algo:ucb} and \ref{algo:ts} satisfy, almost surely, the following :
\beqan
\gamma_t &\ge& \frac{1}{2}\ln(\det(I+\lambda^{-1}K_t)),\\
\gamma_t &\ge& \frac{1}{2}\sum_{s=1}^{t}\ln(1 +\lambda^{-1}\sigma_{s-1}^2(x_s)).
\eeqan
\label{lem:info-theoretic-results}
\end{mylemma}

\begin{proof}
At round $t$ after observing the reward vector $y_{1:t}$ at points $A_t = \lbrace x_1,\ldots,x_t\rbrace \subset D$, the information gain - by the algorithm - about the unknown reward function $f$ is given by the mutual information between $f_{1:t}$ and $y_{1:t}$ sampled at points $A_t$:
\beqn
I(y_{1:t};f_{1:t}) =H(y_{1:t})-H(y_{1:t}\given f_{1:t}),
%\\ &=&  H(y_{1:t}\given A_t)-H(y_{1:t}\given f_{1:t}),
\eeqn
where $y_{1:t}= f_{1:t}+\epsilon_{1:t} = [y_1,\ldots,y_t]^T$, $f_{1:t} = [f(x_1),\ldots,f(x_t)]^T$ and $\epsilon_{1:t}=[\epsilon_1,\ldots,\epsilon_t]^T$. Clearly, given $f_{1:t}$ the randomness - as perceived by the algorithm - in $y_{1:t}$ are only in the noise vector $\epsilon_{1:t}$ and thus 
\beqn
H(y_{1:t}\given f_{1:t}) =  \frac{1}{2}\ln(\det(2\pi e \lambda v^2I))\\ = \frac{t}{2}\log(2\pi e \lambda v^2),
\eeqn
as $\epsilon_{1:t}$ is assumed to follow the distribution $\cN(0,\lambda v^2I)$ and 
$H(\cN(\mu,\Sigma)) = \frac{1}{2}\ln (\det(2\pi e \Sigma))$. Now 
$y_{1:t}$ sampled at points $A_t$ is believed to be distributed as $\cN(0,v^2(K_t+\lambda I))$, which gives
$H(y_{1:t}) =\frac{1}{2}\ln(\det(2\pi e v^2(\lambda I+K_t)))= \frac{t}{2}\log(2\pi e \lambda v^2) + \frac{1}{2}\ln(\det(I+\lambda^{-1}K_t))$,
and therefore 
\beq
I(y_{1:t};f_{1:t}) = \frac{1}{2}\ln(\det(I+\lambda^{-1}K_t)).
\label{eqn:information-gain-batch}
\eeq
%
%\beqan
%H(y_{1:t}) &=& \sum_{s=1}^{t}H(y_s\given y_{A_{s-1}})\\ &=& \sum_{s=1}^{t}H(y_s\given y_{1:(s-1)},A_s).
%\eeqan
Again, conditioned on reward vector $y_{1:s-1}$ observed at points $A_{s-1}$, the reward $y_s$ at round $s$ observed at  $x_s$ is believed to follow the distribution $\cN(\mu_{s-1}(x_s),v^2(\lambda +\sigma_{s-1}^2(x_s)))$, which gives $H(y_s\given y_{1:s-1})=\frac{1}{2}\ln(2\pi ev^2(\lambda +\sigma_{s-1}^2(x_s))) = \frac{1}{2}\ln(2\pi e \lambda v^2)+\frac{1}{2}\ln(1 +\lambda^{-1}\sigma_{s-1}^2(x_s))$. Now by chain rule $H(y_{1:t}) = \sum_{s=1}^{t}H(y_s\given y_{1:s-1})= \frac{t}{2}\ln(2\pi e \lambda v^2) + \frac{1}{2}\sum_{s=1}^{t}\ln(1 +\lambda^{-1}\sigma_{s-1}^2(x_s))$, and therefore
%
%\beqan
%H(y_{1:t}) &=& \frac{1}{2}\sum_{s=1}^{t} \log(2\pi ev^2(1+\sigma_{s-1}^2(x_s)))\\ &=& \frac{t}{2}\log(2\pi e v^2) + \frac{1}{2}\sum_{s=1}^{t}\log(1+\sigma_{s-1}^2(x_s)),
%\eeqan
%So we have for Algorithm \ref{algo:ts},
%\beqa
%I(y_{1:t};f_{1:t}\given A_t) &=& \frac{1}{2}\sum_{s=1}^{t}\log(1+\frac{v_s^2}{v^2}\sigma_{s-1}^2(x_s))\nonumber \\ &\ge &\frac{1}{2}\sum_{s=1}^{t}\log(1+\sigma_{s-1}^2(x_s)),
%\label{eqn:information-gain-ts}
%\eeqa
%as in this case $v = v_1 \le v_t$ for all $t$. Similarly for Algorithm \ref{algo:ucb}, as posterior over $y_s$ is $N(\mu_{s-1}(x_s),v^2+v^2\sigma_{s-1}^2(x_s))$, we have
\beq
I(y_{1:t};f_{1:t}) = \frac{1}{2}\sum_{s=1}^{t}\ln(1 +\lambda^{-1}\sigma_{s-1}^2(x_s)).
\label{eqn:information-gain-online}
\eeq
Now $I(y_{1:t};f_{1:t})$ is a function of $A_t \subset D$, the random points chosen by the algorithm and thus
\beqn
I(y_{1:t};f_{1:t}) \le \max\limits_{A \subset D : \abs{A}=t} I(y_A;f_A) = \gamma_t, \;\; \text{a.s.},
\eeqn
Now the proof follows from Equation  \ref{eqn:information-gain-batch} and \ref{eqn:information-gain-online}.
\end{proof}
\begin{mylemma}
\label{lem:bound-sum-sd} Let $x_1, \ldots x_t$ be the points selected by the algorithms. The sum of predictive standard deviation at those points can be expressed in terms of the maximum information gain. More precisely,
\beqn
\sum_{t=1}^{T}\sigma_{t-1}(x_t)\le \sqrt{4(T+2)\gamma_T}.
\eeqn
\end{mylemma}
\begin{proof}
First note that, by Cauchy-Schwartz inequality, $\sum_{t=1}^{T}\sigma_{t-1}(x_t)\le \sqrt{T\sum_{t=1}^{t}\sigma_{t-1}^2(x_t)}$. Now since $0 \le \sigma_{t-1}^2(x)\le 1$ for all $x \in D$ and by our choice of $\lambda = 1+\eta, \eta \ge 0$, we have $\lambda^{-1}\sigma_{t-1}^2(x_t) \le 2\ln(1+\lambda^{-1}\sigma_{t-1}^2(x_t))$, where in the last inequality we used the fact that for any $0 \le \alpha \le 1$, $\ln(1+\alpha) \ge \alpha/2$. Thus we get $\sigma_{t-1}^2(x_t) \le 2\lambda\ln(1+\lambda^{-1}\sigma_{t-1}^2(x_t))$. This implies
\beqn
\sum_{t=1}^{T}\sigma_{t-1}(x_t)\le \sqrt{2T\sum_{t=1}^{T}\lambda \ln(1+\lambda^{-1}\sigma_{t-1}^2(x_t))}\le \sqrt{4T\lambda \sum_{t=1}^{T}\frac{1}{2}\ln(1+\sigma_{t-1}^2(x_t))} \le \sqrt{4T(1+\eta)\gamma_T},
\eeqn
where the last inequality follows from Lemma \ref{lem:info-theoretic-results}. Now the result follows by choosing $\eta = 2/T$.
\end{proof}
\section*{C. Proof of Theorem \ref{thm:true-function-bound}}

First define $\phi(x)$ as $k(x,\cdot)$, where $\phi : \Real^d\ra H$ maps any point $x$ in the primal space $\Real^d$ to the RKHS $H$ associated with kernel function $k$. For any two functions $g,h \in H$, define the inner product $\inner{g}{h}_k$ as $g^Th$ and the RKHS norm $\norm{g}_k$ as $\sqrt{g^Tg}$. Now as the unknown reward function $f$ lies in the RKHS $H_k(D)$, these definitions along with reproducing property of the RKHS imply $f(x)=\inner{f}{k(x,\cdot)}_k=\inner{f}{\phi(x)}_k=f^T\phi(x)$ and
$k(x,x')=\inner{k(x,\cdot)}{k(x',\cdot)}_k= \inner{\phi(x)}{\phi(x')}_k=\phi(x)^T\phi(x')$ for all $x,x'\in D$. Now defining $\Phi_t \bydef\big[\phi(x_1)^T,\hdots,\phi(x_t)^T\big]^T$, we get the kernel matrix $K_t = \Phi_t\Phi_t^T$, $k_t(x)= \Phi_t\phi(x)$ for all $x\in D$ and $f_{1:t}=\Phi_tf$.
Since the matrices $(\Phi_t^T\Phi_t + I)$ and $(\Phi_t\Phi_t^T + \lambda I)$ are strictly positive definite and $(\Phi_t^T\Phi_t + \lambda I)\Phi_t^T = \Phi_t^T(\Phi_t\Phi_t^T + \lambda I)$, we have
\beq
\Phi_t^T(\Phi_t\Phi_t^T + \lambda I)^{-1} = (\Phi_t^T\Phi_t + \lambda I)^{-1}\Phi_t^T.
\label{eqn:dim-change}
\eeq
Also from the definitions above $(\Phi_t^T\Phi_t+\lambda I)\phi(x)=\Phi_t^Tk_t(x) + \lambda \phi(x)$, and thus from \ref{eqn:dim-change} we deduce that
\beqn
\phi(x)=\Phi_t^T(\Phi_t\Phi_t^T+\lambda I)^{-1}k_t(x)+\lambda (\Phi_t^T\Phi_t + \lambda I)^{-1}\phi(x),
\eeqn
which gives
\beqn
\phi(x)^T\phi(x)= k_t(x)^T(\Phi_t\Phi_t^T+\lambda I)^{-1}k_t(x)+\lambda \phi(x)^T(\Phi_t^T\Phi_t + \lambda I)^{-1}\phi(x).
\eeqn
This implies
\beq
\lambda \phi(x)^T(\Phi_t^T\Phi_t + \lambda I)^{-1}\phi(x) = k(x,x) -k_t(x)^T(K_t+\lambda I)^{-1}k_t(x) = \sigma_t^2(x)
\label{eqn:variance}
\eeq
Now observe that
\beqan
\abs{f(x)-k_t(x)^T(K_t+\lambda I)^{-1}f_{1:t}} &=& \abs{\phi(x)^T f- \phi(x)^T\Phi_t^T(\Phi_t\Phi_t^T+\lambda I)^{-1}\Phi_t f}\\
&=& \abs{\phi(x)^T f-\phi(x)^T(\Phi_t^T\Phi_t + \lambda I)^{-1}\Phi_t^T\Phi_t f}\\
&=& \abs{\lambda \phi(x)^T(\Phi_t^T\Phi_t+\lambda I)^{-1}f}\\
&\le& \norm{\lambda(\Phi_t^T\Phi_t+\lambda I)^{-1}\phi(x)}_k\norm{f}_k\\
&=& \norm{f}_k \sqrt{\lambda \phi(x)^T(\Phi_t^T\Phi_t+\lambda I)^{-1}\lambda I(\Phi_t^T\Phi_t+ \lambda I)^{-1}\phi(x)}\\
&\le & B \sqrt{\lambda \phi(x)^T(\Phi_t^T\Phi_t+\lambda I)^{-1}(\Phi_t^T\Phi_t+\lambda I)(\Phi_t^T\Phi_t+\lambda I)^{-1}\phi(x)}\\
&=&B\; \sigma_t(x), 
\eeqan
where the second equality uses \ref{eqn:dim-change}, the first inequality is by Cauchy-Schwartz and the final equality is from \ref{eqn:variance}. Again see that
\beqan
\abs{k_t(x)^T(K_t+\lambda I)^{-1}\eps_{1:t}} &=& \abs{\phi(x)^T\Phi_t^T(\Phi_t\Phi_t^T + \lambda I)^{-1}\eps_{1:t}}\\
&=& \abs{\phi(x)^T(\Phi_t^T\Phi_t + \lambda I)^{-1}\Phi_t^T\eps_{1:t}} \\
&\le &\norm{(\Phi_t^T\Phi_t + \lambda I)^{-1/2}\phi(x)}_k \norm{(\Phi_t^T\Phi_t + \lambda I)^{-1/2}\Phi_t^T\eps_{1:t}}_k\\
&= &\sqrt{\phi(x)^T(\Phi_t^T\Phi_t + \lambda I)^{-1}\phi(x)}\sqrt{(\Phi_t^T\eps_{1:t})^T(\Phi_t^T\Phi_t + \lambda I)^{-1}\Phi_t^T\eps_{1:t}}\\
&=&\lambda^{-1/2}\sigma_t(x)\sqrt{\eps_{1:t}^T\Phi_t\Phi_t^T(\Phi_t\Phi_t^T +\lambda I)^{-1}\epsilon_{1:t}}\\
&=&\lambda^{-1/2}\sigma_t(x)\sqrt{\eps_{1:t}^T K_t(K_t+\lambda I)^{-1}\eps_{1:t}}
\eeqan
where the second equality is from \ref{eqn:dim-change}, the first inequality is by Cauchy-Schwartz and the fourth inequality uses both \ref{eqn:dim-change} and \ref{eqn:variance}.
%and the final inequality holds as .
Now recall that, at round $t$, the posterior mean function $\mu_t(x) = k_t(x)^T(K_t + \lambda I)^{-1}y_{1:t} = k_t(x)^T(K_t +  \lambda I)^{-1}(f_{1:t}+\epsilon_{1:t})$, where $f_{1:t}=\big[f(x_1),\ldots,f(x_t)\big]^T$ and $\epsilon_{1:t} = \big[\eps_1,\ldots,\eps_t\big]^T$. Thus we have
\beqan
\abs{\mu_t(x)-f(x)} &\le & \abs{k_t(x)^T(K_t+\lambda I)^{-1}\eps_{1:t}} + \abs{f(x)-k_t(x)^T(K_t+\lambda I)^{-1}f_{1:t}}\\&\le& \sigma_t(x)\Big(B + (1+\eta)^{-1/2}\sqrt{\eps_{1:t}^T K_t(K_t+(1+\eta) I)^{-1}\eps_{1:t}}\Big),
\eeqan
where we have used $\lambda = 1+\eta$, where $\eta \ge 0$ as stated in Theorem \ref{thm:self-normalized-bound}. Now observe that when $K$ is invertible, $K(K+I)^{-1}=((K+I)K^{-1})^{-1}=(I+K^{-1})^{-1}$. Using $K=K_t+\eta I$, we get 
\beqn
(K_t+\eta I)(K_t+(1+\eta)I)^{-1}=((K_t+\eta I)^{-1}+I)^{-1}.
\eeqn
Now see that
\beqn
\epsilon_{1:t}^TK_t(K_t+(1+\eta)I)^{-1}\epsilon_{1:t} \le \epsilon_{1:t}^T(K_t+\eta I)(K_t+(1+\eta)I)^{-1}\epsilon_{1:t}=\epsilon_{1:t}^T((K_t+\eta I)^{-1}+I)^{-1}\epsilon_{1:t}
\eeqn
Now using Theorem \ref{thm:self-normalized-bound}, for any $\delta \in (0,1)$, with probability at least $1-\delta$, $\forall t \ge 0, \forall x \in D$, we obtain
\beqn
\abs{\mu_t(x)-f(x)}\le  \sigma_t(x)\Big(B+\norm{\epsilon_{1:t}}_{(K_t+\eta I)^{-1}+I)^{-1}}\Big)\\
 \le  \sigma_t(x)\Big(B + R\sqrt{2\ln\frac{\sqrt{\det((1+\eta)I+K_t)}}{\delta}}\Big).\\
\eeqn
Now observe that $\det((1+\eta)I+K_t)=\det(I+(1+\eta)^{-1}K_t)\det((1+\eta)I)$. Thus we have
\beqn
\ln(\det((1+\eta)I+K_t))=\ln(\det(I+(1+\eta)^{-1}K_t))+t\ln(1+\eta) \le 2\gamma_t + \eta t,
\eeqn
from lemma \ref{lem:info-theoretic-results}. Now choosing $\eta = 2/T$ we have $\abs{\mu_t(x)-f(x)}\le \sigma_t(x)\Big(B + R\sqrt{2\big(\gamma_t+ 1+ \ln(1/\delta)\big)}\Big)$ and hence the result follows.
\qed
\section*{D. Analysis of IGP-UCB (Theorem \ref{thm:regret-bound-UCB})}
Observe that at each round $t \ge 1$, by the choice of $x_t$ in Algorithm \ref{algo:ucb}, we have $\mu_{t-1}(x_t)+\beta_t\sigma_{t-1}(x_t) \ge \mu_{t-1}(x^\star)+\beta_t\sigma_{t-1}(x^\star)$ and from Lemma \ref{thm:true-function-bound}, we have $f(x^\star) \le \mu_{t-1}(x^\star)+\beta_t\sigma_{t-1}(x^\star)$ and $\mu_{t-1}(x_t)-f(x_t) \le \beta_t\sigma_{t-1}(x_t)$. Therefore for all $t \ge 1$ with probability at least $1-\delta$,
\beqan
r_t &=& f(x^\star) - f(x_t)\\
&\le & \beta_t\sigma_{t-1}(x_t)+\mu_{t-1}(x_t)-f(x_t)\\
&\le & 2\beta_t\sigma_{t-1}(x_t), 
\eeqan
and hence $\sum\limits_{t=1}^{T}r_t \le 2\beta_T\sum\limits_{t=1}^{T}\sigma_{t-1}(x_t)$. Now from Lemma \ref{lem:bound-sum-sd}, $\sum\limits_{t=1}^{T}\sigma_{t-1}(x_t) = O(\sqrt{T\gamma_T}) $ and by definition $\beta_T \le B + R\sqrt{2(\gamma_T+1+\ln(1/\delta))}$. Hence with probability at least $1-\delta$,
\beqn
R_T = \sum\limits_{t=1}^{T}r_t = O\Big(B\sqrt{T\gamma_T}+\sqrt{T\gamma_T(\gamma_T+\ln(1/\delta))}\Big),
\eeqn
and thus with high probability,
\beqn
R_T = O\Big(\sqrt{T}(B\sqrt{\gamma_T}+\gamma_T)\Big).
\eeqn
\section*{E. Analysis of GP-TS (Theorem \ref{thm:regret-bound-TS})}
\begin{mylemma}
For any $\delta \in (0,1)$ and any finite subset $D'$ of $D$,
\beqn
\mathbb{P}\Big[\forall x \in D',\abs{f_t(x)-\mu_{t-1}(x)}\le v_t\sqrt{2\ln(\abs{D'}t^2)}\;\sigma_{t-1}(x) \given \cH_{t-1}\Big]\ge 1- 1/t^2,
\eeqn
for all possible realizations of history $\cH_{t-1}$.
\label{lem:gaussian-concentration}
\end{mylemma}
\begin{proof}
Fix $x \in D$ and $t \ge 1$. Given history $\cH_{t-1}$, $f_t(x)\sim \cN( \mu_{t-1}(x),v_t^2\sigma_{t-1}^2(x))$. Thus using Lemma $B4$ of \cite{hoffman2013exploiting}, for any $\delta \in (0,1)$, with probability at least $1-\delta$
\beqn
\abs{f_t(x)-\mu_{t-1}(x)} \le \sqrt{2\ln(1/\delta)}\;v_t\sigma_{t-1}(x),
\eeqn
and now applying union bound,
\beqn
\abs{f_t(x)-\mu_{t-1}(x)} \le v_t\sqrt{2\ln(\abs{D'}/\delta)}\;\sigma_{t-1}(x) \;\; \forall x \in D'
\eeqn
holds with probability at least $1-\delta$, given any possible realizations of history $\cH_{t-1}$. Now setting $\delta = 1/t^2$, the result follows.
\end{proof}
\begin{mydefinition}
Define For all $t \ge 1$, $\tilde{c}_t= \sqrt{4\ln t + 2d\ln(BLrdt^2)}$ and $c_t=v_t(1+\tilde{c}_t)$, where $v_t = B + R\sqrt{2(\gamma_{t-1}+ 1+\ln(2/\delta))}$. Clearly, $c_t$ increases with $t$.
\label{def:constants}
\end{mydefinition}
\begin{mydefinition}
Define $E^f(t)$ as the event that for all $x\in D$,
\beqn
\abs{\mu_{t-1}(x)-f(x)} \le v_t\sigma_{t-1}(x),
\eeqn
and $E^{f_t}(t)$ as the event that for all $x\in D_t$,
\beqn
\abs{f_t(x)-\mu_{t-1}(x)} \le v_t \tilde{c}_t\sigma_{t-1}(x).
\eeqn
\label{def:two-events}
\end{mydefinition}
\begin{mydefinition}
Define the set of saturated points $S_t$ in discretization $D_t$ at round $t$ as
\beqn
S_t \bydef \lbrace x\in D_t: \Delta_t(x) > c_t\sigma_{t-1}(x)\rbrace,
\eeqn
where $\Delta_t(x) \bydef f([x^\star]_t)-f(x)$, the difference between function values at the closest point to $x^\star$ in $D_t$ and at $x$. Clearly $\Delta_t([x^\star]_t) = 0$ for all $t$, and hence $[x^\star]_t\in D_t$ is unsaturated at every $t$.
\label{def:saturated-arms}
\end{mydefinition}
\begin{mydefinition}
Define filtration $\cF^{'}_{t-1}$ as the history until time $t$, i.e., $\cF^{'}_{t-1}=\cH_{t-1}$. By definition, $\cF^{'}_1\subseteq \cF{'}_2 \subseteq \cdots$. Observe that given $\cF^{'}_{t-1}$, the set $S_t$ and the event $E^f(t)$ are completely deterministic.
\label{def:filtration}
\end{mydefinition}
\begin{mylemma}
Given any $\delta \in (0,1)$, $\prob{\forall t \ge 1, E^f(t)}\ge 1-\delta/2 $ and for all possible filtrations $\cF^{'}_{t-1}$, $\prob{E^{f_t}(t)\given \cF^{'}_{t-1}} \ge 1-1/t^2$.
\label{lem:event-concentration}
\end{mylemma}
\begin{proof}
The probability bound for the event $E^f(t)$ follows from Theorem \ref{thm:true-function-bound} by replacing $\delta$ with $\frac{\delta}{2}$ and for the event $E^{f_t}(t)$ follows from Lemma \ref{lem:gaussian-concentration} by setting $D'=D_t$ and $\cH_{t-1}=\cF^{'}_{t-1}$.
\end{proof}
\begin{mylemma}[Gaussian Anti-concentration]
For a Gaussian random variable $X$ with mean $\mu$ and standard deviation $\sigma$, for any $\beta > 0$,
\beqn
\prob{\frac{X-\mu}{\sigma} > \beta} \ge \frac{e^{-\beta^2}}{4\sqrt{\pi}\beta}.
\eeqn
\label{anti-concentration}
\end{mylemma}
\begin{mylemma}
For any filtration $\cF^{'}_{t-1}$ such that $E^f(t)$ is true,
\beqn
\prob{f_t(x)>f(x)\given \cF^{'}_{t-1}} \ge p,
\eeqn
for any $x \in D$, where $p = \frac{1}{4e\sqrt{\pi}}$.
\label{lem:compare-sample-with-original}
\end{mylemma}
\begin{proof}
Fix any $x\in D$. Given filtration $\cF^{'}_{t-1}$, $f_t(x)$ is a Gaussian random variable  with mean $\mu_{t-1}(x)$ and standard deviation $v_t\sigma_{t-1}(x)$ and since event $E^f(t)$ is true, $\abs{\mu_{t-1}(x)-f(x)} \le c_{1,t}\sigma_{t-1}(x)$. Now using the anti-concentration inequality in Lemma \ref{anti-concentration}, we have
\beqan
\prob{f_t(x)>f(x)\given \cF^{'}_{t-1}} &=& \prob{\frac{f_t(x)-\mu_{t-1}(x)}{v_t \sigma_{t-1}(x)}>\frac{f(x)-\mu_{t-1}(x)}{v_t \sigma_{t-1}(x)}\given \cF^{'}_{t-1}}\\
&\ge&\prob{\frac{f_t(x)-\mu_{t-1}(x)}{v_t \sigma_{t-1}(x)}>\frac{\abs{f(x)-\mu_{t-1}(x)}}{v_t \sigma_{t-1}(x)}\given \cF^{'}_{t-1}}\\
&\ge&\frac{1}{4\sqrt{\pi}\beta_t}e^{-\theta_t^2},
\eeqan
where, from Definition \ref{def:two-events}, $\theta_t = \frac{\abs{f(x)-\mu_{t-1}(x)}}{v_t \sigma_{t-1}(x)} \le 1$. Therefore $\prob{f_t(x)>f(x)\given \cF^{'}_{t-1}} \ge \frac{1}{4e\sqrt{\pi}}$, and hence the result follows.
\end{proof}
\begin{mylemma}
For any filtration $\cF^{'}_{t-1}$ such that $E^f(t)$ is true,
\beqn
\prob{x_t\in D_t\setminus S_t\given \cF^{'}_{t-1}} \ge p-1/t^2.
\eeqn
\label{lem:prob-playing-saturated-arms}
\end{mylemma}
\begin{proof}
At round $t$ our algorithm chooses the point $x_t\in D_t$, at which the highest value of $f_t$, within current decision set $D_t$, is attained. Now if $f_t([x^\star]_t)$ is greater than $f_t(x)$ for all saturated points at round $t$, i.e.,$f_t([x^\star]_t) > f_t(x), \forall x \in S_t$, then one of the unsaturated points (which includes $[x^\star]_t$) in $D_t$ must be played and hence $x_t \in D_t\setminus S_t$. This implies
\beq
\prob{x_t\in D_t\setminus S_t\given \cF^{'}_{t-1}}  \ge \prob{f_t([x^\star]_t) > f_t(x), \forall x \in S_t\given \cF^{'}_{t-1}}.
\label{eqn:prob-saturated-arms}
\eeq
Now form Definition \ref{def:saturated-arms}, $\Delta_t(x) > c_t\sigma_{t-1}(x)$, for all $x\in S_t$. Also if both the events $E^f(t)$ and $E^{f_t}(t)$ are true, then from Definition \ref{def:constants} and \ref{def:two-events}, $f_t(x) \le f(x) + c_t\sigma_{t-1}(x)$, for all $x \in D_t$. Thus for all $x\in S_t$, $f_t(x) < f(x) + \Delta_t(x)$. 
Therefore, for any filtration $\cF^{'}_{t-1}$ such that $E^f(t)$ is true, either $E^{f_t}(t)$ is false, or else for all $x \in S_t$, 
$f_t(x) < f([x^\star]_t)$. Hence, for any $\cF^{'}_{t-1}$ such that $E^f(t)$ is true,
\beqn
\prob{f_t([x^\star]_t) > f_t(x), \forall x \in S_t\given \cF^{'}_{t-1}}\\ \ge  \prob{f_t([x^\star]_t) > f([x^\star]_t)\given \cF_{t-1}} - \prob{\ol{E^{f_t}(t)}\given \cF^{'}_{t-1}} \\
\ge p - 1/t^2,
\eeqn
where we have used Lemma \ref{lem:event-concentration} and Lemma \ref{lem:compare-sample-with-original}. Now the proof follows from Equation \ref{eqn:prob-saturated-arms}.
\end{proof}
\begin{mylemma}
For any filtration $\cF^{'}_{t-1}$ such that $E^f(t)$ is true,
\beqn
\expect{r_t\given \cF^{'}_{t-1}} \le \frac{11c_t}{p}\expect{\sigma_{t-1}(x_t)\given \cF^{'}_{t-1}} + \frac{2B+1}{t^2},
\eeqn
where $r_t$ is the instantaneous regret at round $t$. 
\label{lem:bound-on-instantaneous-regret}
\end{mylemma}
\begin{proof}
Let $\bar{x}_t$ be the unsaturated point in $D_t$ with smallest $\sigma_{t-1}(x)$, i.e.,
\beq
\bar{x}_t = \argmin\limits_{x \in D_t\setminus S_t}\sigma_{t-1}(x).
\label{eqn:smallest-sd}
\eeq
Since $\sigma_{t-1}(\cdot)$ and $S_t$ are deterministic given $\cF^{'}_{t-1}$, so is $\bar{x}_t$. Now for any $\cF^{'}_{t-1}$ such that $E^f(t)$ is true,
\beqa
\expect{\sigma_{t-1}(x_t)\given \cF^{'}_{t-1}} &\ge& \expect{\sigma_{t-1}(x_t)\given \cF^{'}_{t-1},x_t \in D_t\setminus S_t}\prob{x_t\in D_t\setminus S_t\given \cF^{'}_{t-1}}\nonumber\\
&\ge& \sigma_{t-1}(\bar{x}_t)(p-1/t^2),
\label{eqn:expectation-smallest-sd}
\eeqa
where we have used Equation \ref{eqn:smallest-sd} and Lemma \ref{lem:prob-playing-saturated-arms}. Now, if both the events  $E^f(t)$ and $E^{f_t}(t)$ are true, then from Definition \ref{def:constants} and \ref{def:two-events}, $f(x) - c_t\sigma_{t-1}(x) \le f_t(x) \le f(x) + c_t\sigma_{t-1}(x)$, for all $x \in D_t$. Using this observation along with Definition \ref{def:saturated-arms} and the facts that $f_t(x_t) \ge f_t(x)$ for all $x \in D_t$ and $\bar{x}_t \in D_t\setminus S_t$, we have
\beqan
\Delta_t(x_t) &=& f([x^\star]_t)-f(\bar{x}_t)+f(\bar{x}_t)-f(x_t)\\
&\le& \Delta_t(\bar{x}_t)+ f_t(\bar{x}_t)+c_t\sigma_{t-1}(\bar{x}_t)-f_t(x_t)+c_t\sigma_{t-1}(x_t)\\
&\le& c_t\sigma_{t-1}(\bar{x}_t) + c_t\sigma_{t-1}(\bar{x}_t)+c_t\sigma_{t-1}(x_t)\\
&\le& c_t\big(2\sigma_{t-1}(\bar{x}_t)+c_t\sigma_{t-1}(x_t)\big).
\eeqan
Therefore, for any $\cF^{'}_{t-1}$ such that $E^f(t)$ is true, either $\Delta_t(x_t) \le c_t\big(2\sigma_{t-1}(\bar{x}_t)+c_t\sigma_{t-1}(x_t)\big)$, or $E^{f_t}(t)$ is false. Now from our assumption of bounded variance, for all $x \in D$, $\abs{f(x)} \le \norm{f}_k k(x,x) \le B$, and hence $\Delta_t(x) \le 2\sup\limits_{x\in D} \abs{f(x)} \le 2B$. Thus, using Equation \ref{eqn:expectation-smallest-sd}, we get
\beqa
\expect{\Delta_t(x_t)\given \cF^{'}_{t-1}} &\le& \expect{c_t\big(2\sigma_{t-1}(\bar{x}_t)+c_t\sigma_{t-1}(x_t)\big)\given \cF^{'}_{t-1}}+ 2B\prob{\ol{E^{f_t}(t)}\given \cF^{'}_{t-1}}\nonumber\\
&\le& \frac{2c_t}{p-1/t^2} \expect{\sigma_{t-1}(x_t)\given \cF^{'}_{t-1}} + c_t \expect{\sigma_{t-1}(x_t)\given \cF^{'}_{t-1}} + \frac{2B}{t^2}\nonumber\\
&\le& \frac{11c_t}{p}\expect{\sigma_{t-1}(x_t)\given \cF^{'}_{t-1}} +\frac{2B}{t^2},
\label{eqn:expectation-arm-played}
\eeqa
where in the last inequality we used that $1/(p-1/t^2) \le 5/p$, which holds trivially for $t \le 4$ and also holds for $t\ge 5$, as $t^2 > 5e\sqrt{\pi}$. Now using Equation \ref{eqn:lipschitz}, we have the instantaneous regret at round $t$,
\beqn
r_t = f(x^\star)-f([x^\star]_t)+f([x^\star]_t)-f(x_t)
\le \frac{1}{t^2} + \Delta_t(x_t),
\eeqn
and then taking conditional expectation on both sides, the result follows from Equation \ref{eqn:expectation-arm-played}.
\end{proof}
\begin{mydefinition}
Let us define $Y_0 = 0$, and for all $t=1,\ldots,T$:
\beqan
\bar{r}_t &=& r_t \cdot \indic{E^f(t)},\\
X_t &=& \bar{r}_t - \frac{11c_t}{p}\sigma_{t-1}(x_t)-\frac{2B+1}{t^2},\\
Y_t &=& \sum_{s=1}^{t}X_s.
\eeqan
\label{def:regret}
\end{mydefinition}
\begin{mydefinition}
A sequence of random variables $(Z_t;t\ge 0)$ is called a super-martingale corresponding to a filtration $\cF_t$, if for all $t$, $Z_t$ is $\cF_t$-measurable, and for $t\ge 1$,
\beqn
\expect{Z_t\given \cF_{t-1}} \le Z_{t-1}.
\eeqn
\label{def:supermartingale}
\end{mydefinition}
\begin{mylemma}[Azuma-Hoeffding Inequality] 
If a super-martingale $(Z_t;t \ge 0)$, corresponding to filtration $\cF_t$, satisfies $\abs{Z_t - Z_{t-1}} \le \alpha_t$ for some constant $\alpha_t$, for all $t = 1,\ldots,T$, then for any $\delta \ge 0$,
\beqn
\prob{Z_T - Z_0 \le \sqrt{2\ln(1/\delta)\sum_{t=1}^{T}\alpha_t^2}\;} \ge 1-\delta. 	
\eeqn
\label{lem:azuma-hoeffding}
\end{mylemma}
\begin{mylemma}
$(Y_t;t = 0,...,T)$ is a super-martingale process with respect to filtration $\cF^{'}_t$.
\label{lem:supermartingale}
\end{mylemma}
\begin{proof} 
From Definition \ref{def:supermartingale}, we need to prove that for all $t \in \lbrace 1,\ldots,T\rbrace$ and any possible $\cF^{'}_{t-1}$, $\expect{Y_t - Y_{t-1}\given \cF^{'}_{t-1}} \le 0$, i.e.
\beq
\expect{\bar{r}_t|\cF^{'}_{t-1}} \le  \frac{11c_t}{p}\expect{\sigma_{t-1}(x_t)|\cF^{'}_{t-1}} + \frac{2B+1}{t^2}.
\label{eqn:supermartingale}	
\eeq
Now if $\cF^{'}_{t-1}$ such that $E^f(t)$ is false, then $\bar{r}_t = r_t \cdot \indic{E^f(t)} = 0$, and Equation \ref{eqn:supermartingale} holds trivially. Moreover, for $\cF^{'}_{t-1}$ such that both  $E^f(t)$ is true, Equation \ref{eqn:supermartingale} follows from Lemma \ref{lem:bound-on-instantaneous-regret}.
\end{proof}
\begin{mylemma}
Given any $\delta \in (0,1)$, with probability at least $1-\delta$,
\beqn
R_T=\sum_{t=1}^{T}r(t)=\frac{11c_T}{p}\sum_{t=1}^{T}\sigma_{t-1}(x_t)  +\frac{(2B+1)\pi^2}{6}+\frac{(4B+11)c_T}{p}\sqrt{2T\ln(2/\delta)},
\eeqn
where $T$ is the total number of rounds played.
\label{lem:total-regret}
\end{mylemma}
\begin{proof}
First note that from Definition \ref{def:regret} for all $t=1,\ldots,T$,
\beqn
\abs{Y_t - Y_{t-1}} = \abs{X_t} \le\abs{\bar{r}_t} + \frac{11c_t}{p}\sigma_{t-1}(x_t) + \frac{2B+1}{t^2}.
\eeqn
Now as $\bar{r}_t \le r_t \le 2\sup\limits_{x\in D}\abs{f(x)} \le 2B$ and $\sigma^2_{t-1}(x_t) \le \sigma^2_0(x_t) \le 1$, we have
\beqn
\abs{Y_t - Y_{t-1}} 
\le 2B + \frac{11c_t}{p} + \frac{2B+1}{t^2}
\le \frac{(4B+11)c_t}{p},
\eeqn
which follows from the fact that $2B \le 2Bc_t/p$ and also $(2B+1)/t^2 \le 2Bc_t/p$. Thus, we can apply Azuma-Hoeffding inequality (Lemma \ref{lem:azuma-hoeffding}) to obtain that with probability at least $1-\delta/2$,
\beqan
\sum_{t=1}^{T}\bar{r}_t &\le & \sum_{t=1}^{T}\frac{11c_t}{p}\sigma_{t-1}(x_t) + \sum_{t=1}^{T}\frac{2B+1}{t^2}+ \sqrt{2\ln(2/\delta)\sum_{t=1}^{T}\frac{(4B+11)^2c_t^2}{p^2}}\\
&\le & \frac{11c_T}{p}\sum_{t=1}^{T}\sigma_{t-1}(x_t) + \frac{(2B+1)\pi^2}{6} + \frac{(4B+11)c_T}{p}\sqrt{2T\ln(2/\delta)},
\eeqan
as by definition $c_t \le c_T$ for all $t \in \lbrace 1,\ldots,T\rbrace$.
Now,  as the event $E^f(t)$ holds holds for all $t$ with probability at least $1-\delta/2$ (see Lemma \ref{lem:event-concentration}), then from Definition \ref{def:regret}, $\bar{r}_t=r_t$ for all $t$ with probability at least $1-\delta/2$. Now by applying union bound, the result follows.
\end{proof}
\subsection*{Proof of Theorem \ref{thm:regret-bound-TS}}
From Lemma \ref{lem:bound-sum-sd} we have, $\sum\limits_{t=1}^{T}\sigma_{t-1}(x_t) = O(\sqrt{T\gamma_T})$. Also from Definition \ref{def:constants},
\beqan
C_T &\le& B + R\sqrt{2(\gamma_T+ 1+\ln(2/\delta))} + \Big(B + R\sqrt{2(\gamma_T+ 1+ \ln(2/\delta))}\Big)\sqrt{4\ln T + 2d\ln(BLrdT^2)}\\
%O\Big(B\sqrt{\ln T} + \sqrt{(\gamma_T+\ln(1/\delta))\ln T}\Big)
&=& O\Big(\sqrt{(\gamma_T+\ln(2/\delta))(\ln T+d\ln(BdT))}+B\sqrt{d\ln(BdT)}\Big)\\
&=& O\Big(\sqrt{(\gamma_T+\ln(2/\delta))d\ln(BdT)}\Big).
\eeqan
Hence, from Lemma \ref{lem:total-regret}, with probability at least $1-\delta$,
\beqan
R_T=O\Bigg(\sqrt{(\gamma_T+\ln(2/\delta))d\ln (BdT)}\cdot \Big(\sqrt{T\gamma_T}+B\sqrt{T\ln(2/\delta)}\Big)\Bigg)
\eeqan
and thus with high probability,
\beqan
R_T&=& O\Big(\sqrt{T\gamma_T^2 d\ln (BdT)} + B\sqrt{T\gamma_Td\ln (BdT)}\Big)\\ &=& O\Bigg(\sqrt{Td\ln (BdT)}\Big(B\sqrt{\gamma_T}+\gamma_T\Big) \Bigg). \hspace*{75pt}
\eeqan
\section*{F. Recursive Updates of Posterior Mean and Covariance}

We now describe a procedure to update the posterior mean and covariance function in a recursive fashion through the properties of Schur complement (\citet{zhang2006schur}) rather than evaluating Equation \ref{eqn:mean update} and \ref{eqn:cov update} at each round. Specifically for all $t \ge 1$ we show the following:
\beqa
\mu_t(x) &=& \mu_{t-1}(x) + \dfrac{k_{t-1}(x,x_t)}{\lambda+\sigma^2_{t-1}(x_t)}(y_t - \mu_{t-1}(x_t)),\label{eqn:mean-online}\\
k_t(x,x') &=& k_{t-1}(x,x') - \dfrac{k_{t-1}(x,x_t)k_{t-1}(x_t,x')}{\lambda+ \sigma^2_{t-1}(x_t)},\label{eqn:cov-online}\\
\sigma^2_t(x) &=& \sigma^2_{t-1}(x) - \dfrac{k^2_{t-1}(x,x_t)}{\lambda+ \sigma^2_{t-1}(x_t)}.\label{eqn:sd-online}
\eeqa
These update rules make our algorithms easy to implement and we are not aware of any literature which explicitly states or uses these relations.

First we write the matrix $K_t + \lambda I$ as $\begin{bmatrix}
A & B\\
C & D
\end{bmatrix}$, where $A = K_{t-1} + \lambda I$, $B = k_{t-1}(x_t)$, $C = B^T$ and $D = \lambda+k(x_t,x_t)$. Now using Schur's complement we get
\beqan
\begin{bmatrix}
A & B\\
C & D
\end{bmatrix}^{-1} &=& \begin{bmatrix}
A^{-1}+A^{-1}B\beta CA^{-1} & -A^{-1}B\beta\\
-\beta CA^{-1} & \beta
\end{bmatrix}\\ &=& \begin{bmatrix}
A^{-1}+\beta A^{-1}BB^TA^{-1} & -\beta A^{-1}B\\
-\beta B^TA^{-1} & \beta
\end{bmatrix}\\ &=& \begin{bmatrix}
A^{-1}+\beta \gamma & -\beta \alpha\\
-\beta \alpha^T & \beta
\end{bmatrix},
\eeqan
where $\beta = (D - CA^{-1}B)^{-1} = 1/(D - B^TA^{-1}B)$, $\gamma = A^{-1}BB^TA^{-1}$ and $\alpha = A^{-1}B$. Therefore we have
\beqan
\mu_t(x) &=& k_t(x)^T(K_t +  \lambda I)^{-1}y_{1:t}\\
&=& \begin{bmatrix}
k_{t-1}(x)^T & k(x_t,x)
\end{bmatrix}\begin{bmatrix}
A^{-1}+\beta \gamma & -\beta \alpha\\
-\beta \alpha^T & \beta
\end{bmatrix}\begin{bmatrix}
y_{1:{t-1}}\\
y_t
\end{bmatrix}\\
&=& k_{t-1}(x)^T(A^{-1}+\beta \gamma)y_{1:{t-1}} - \beta k(x_t,x)\alpha^Ty_{1:{t-1}} - \beta y_t \alpha^Tk_{t-1}(x) +\beta y_t k(x_t,x)\\
&=& k_{t-1}(x)^T A^{-1}y_{1:{t-1}} + \beta \Big(k_{t-1}(x)^T\gamma y_{1:{t-1}}-k(x_t,x)\alpha^Ty_{1:{t-1}} - y_t \alpha^Tk_{t-1}(x) + y_t k(x_t,x)\Big),
\eeqan
where
\beqan
k_{t-1}(x)^T A^{-1}y_{1:{t-1}} &=& k_{t-1}(x)^T (K_{t-1}+\lambda I)^{-1}y_{1:{t-1}} = \mu_{t-1}(x),\\
k_{t-1}(x)^T\gamma y_{1:{t-1}} &=& k_{t-1}(x)^T A^{-1}k_{t-1}(x_t)k_{t-1}(x_t)^TA^{-1}y_{1:{t-1}}= \left(k_{t-1}(x_t)^T A^{-1}k_{t-1}(x)\right)\mu_{t-1}(x_t),\\
\alpha^Ty_{1:{t-1}} &=& k_{t-1}(x_t)^TA^{-1}y_{1:{t-1}} = \mu_{t-1}(x_t),\\
\alpha^Tk_{t-1}(x) &=& k_{t-1}(x_t)^T A^{-1}k_{t-1}(x).
\eeqan
Thus we have
\beqan
\mu_t(x) &=& \mu_{t-1}(x) + \beta \Big( k_{t-1}(x_t)^T A^{-1}k_{t-1}(x)\left(\mu_{t-1}(x_t)-y_t \right)+ k(x_t,x)\left(y_t - \mu_{t-1}(x_t)\right)\Big)\\
&=& \mu_{t-1}(x) + \beta\left(y_t - \mu_{t-1}(x_t)\right)\left(k(x_t,x) - k_{t-1}(x_t)^T A^{-1}k_{t-1}(x)\right)\\
&=& \mu_{t-1}(x) + \beta k_{t-1}(x_t,x)(y_t - \mu_{t-1}(x_t)).
\eeqan
Now as
\beqn
D - B^TA^{-1}B = \lambda+k(x_t,x_t) - k_{t-1}(x_t)^T(K_{t-1}+\lambda I)^{-1}k_{t-1}(x_t) = \lambda +\sigma_{t-1}^2(x_t),
\eeqn
putting $\beta = 1/(\lambda+\sigma^2_{t-1}(x_t))$, we obtain Equation \ref{eqn:mean-online}. Again observe that
\beqan
&& k_t(x,x')\\&=& k(x,x') - k_t(x)^T(K_t + \lambda I)^{-1}k_t(x')\\
&=& k(x,x') - k_{t-1}(x)^TA^{-1}k_{t-1}(x')\\ &&+ \beta \Big(k_{t-1}(x)^T\gamma k_{t-1}(x')- k(x_t,x)\alpha^Tk_{t-1}(x')-k(x_t,x')\alpha^Tk_{t-1}(x)+k(x_t,x)k(x_t,x')\Big).
\eeqan
Now we have
\beqn
k(x,x') - k_{t-1}(x)^TA^{-1}k_{t-1}(x')= k(x,x') - k_{t-1}(x)^T(K_{t-1} + \lambda I)^{-1}k_{t-1}(x') = k_{t-1}(x,x'),
\eeqn
also
\beqan
&& k_{t-1}(x)^T\gamma k_{t-1}(x') - k(x_t,x)\alpha^Tk_{t-1}(x') \\&=& k_{t-1}(x)^TA^{-1}k_{t-1}(x_t)k_{t-1}(x_t)^TA^{-1}k
_{t-1}(x') - k(x_t,x)k_{t-1}(x_t)^TA^{-1}k_{t-1}(x')\\
&=& (k_{t-1}(x)^TA^{-1}k_{t-1}(x_t)-k(x_t,x))k_{t-1}(x_t)^TA^{-1}k_{t-1}(x')\\
&=& -k_{t-1}(x_t,x)k_{t-1}(x_t)^TA^{-1}k_{t-1}(x'),
\eeqan
and
\beqan
k(x_t,x)k(x_t,x') - k(x_t,x')\alpha^Tk_{t-1}(x) &=& k(x_t,x')(k(x_t,x)-k_{t-1}(x_t)^TA^{-1}k_{t-1}(x))\\
&=& k(x_t,x')k_{t-1}(x_t,x).
\eeqan
Putting all these together we get
\beqan
k_t(x,x') &=& k_{t-1}(x,x') - \beta \Big(k(x_t,x')k_{t-1}(x_t,x)- k_{t-1}(x_t,x)k_{t-1}(x_t)^TA^{-1}k_{t-1}(x')\Big)\\
&=&k_{t-1}(x,x') - \beta \Big(k_{t-1}(x_t,x)\Big(k(x_t,x')-k_{t-1}(x_t)^TA^{-1}k_{t-1}(x')\Big)\Big)\\
&=&k_{t-1}(x,x') - \beta k_{t-1}(x_t,x)k_{t-1}(x_t,x').
\eeqan
Now Equation \ref{eqn:cov-online} and \ref{eqn:sd-online} follows by using $\beta = 1/ (\lambda+\sigma^2_{t-1}(x))$ 
and $\sigma_t^2(x) = k_t(x,x)$.
\qed

\end{document}